\documentclass{article}
\usepackage{amsfonts,amsmath,amssymb,amsthm}
\usepackage{enumitem}

\usepackage{multirow}
\usepackage{array,booktabs,colortbl,tabularx,multirow}
\usepackage{geometry}
\usepackage{algorithm}
\usepackage{algpseudocode}

\usepackage[numbers]{natbib}

\usepackage[utf8]{inputenc} 
\usepackage[T1]{fontenc}    
\usepackage[hidelinks]{hyperref}       
\usepackage{url}            

\usepackage{bbm}
\usepackage{nicefrac}       

\usepackage{graphicx}
\usepackage{caption}
\usepackage{subcaption}
\usepackage{algorithm}
\usepackage[colorinlistoftodos]{todonotes}

\usepackage{tikz}
\usepackage{pgfplots}

\usepackage{xcolor}

\definecolor{lightgray}{gray}{0.9}

\usepackage{microtype}      
\usepackage{afterpage}

\usepackage{lipsum}
\newtheorem{thm}{Theorem}

\usepackage{xr}

\usepackage{amsmath}
\usepackage{graphicx}
\usepackage{booktabs}
\usepackage{multirow}
\usepackage{}
\usepackage{xcolor}
\usepackage{makecell}
\usepackage{pifont}

\makeatletter
\def\BState{\State\hskip-\ALG@thistlm}
\makeatother

\usepackage{fullpage}

\usepackage{verbatim}

\usepackage{threeparttable}
\usepackage{algorithm}

\usepackage[usestackEOL]{stackengine}

\DeclareMathAlphabet{\mathpzc}{OT1}{pzc}{m}{it}
\usepackage{cleveref}
\usepackage{hyperref}[]
\hypersetup{
	colorlinks=true,
	linkcolor=blue,
	filecolor=magenta,      
	urlcolor=cyan,
	citecolor=blue,
}

\usepackage{graphicx}
\usepackage{longtable}

\newtheorem{definition}{Definition}
\newtheorem{theorem}{Theorem}

\newtheorem{proposition}[theorem]{Proposition}

\newtheorem{remark}{Remark}

\newtheorem{example}{Example}

\usepackage{amsmath}

\DeclareMathOperator*{\argminB}{argmin}   

\usepackage{xspace}
\usepackage[T1]{fontenc}
\usepackage[utf8]{inputenc}
\usepackage{authblk}

\title{Variable Selection Methods for Multivariate, Functional, and Complex Biomedical Data in the AI Age}
\author{Marcos Matabuena$^{1}$ \\
    $^{1}$Universidade de Santiago de Compostela and Harvard University 
  \\
}
\date{\today}
\begin{document}
\maketitle
\begin{abstract}
Many problems within personalized medicine and digital health rely on the analysis of continuous-time functional biomarkers and other complex data structures emerging from high-resolution patient monitoring. In this context, this work proposes new optimization-based variable selection methods for multivariate, functional, and even more general outcomes in metrics spaces based on best-subset selection.
Our framework applies to several types of regression models, including linear, quantile, or non parametric additive models, and to a broad range of random responses, such as univariate, multivariate Euclidean data, functional, and even random graphs. Our analysis demonstrates that our proposed methodology outperforms state-of-the-art methods in accuracy and, especially, in speed—achieving several orders of magnitude improvement over competitors across various type of statistical responses as  the case of mathematical functions. While our framework is general and is not designed for a specific regression and  scientific problem, the article is self-contained and focuses on biomedical applications. In the clinical areas, serves as a valuable resource for professionals in biostatistics, statistics, and artificial intelligence interested in variable selection problem in this new technological AI-era.
\end{abstract}
\section{Introduction}
Recent technological advances have enabled the monitoring of biological systems at an unprecedented resolution, leading to the routine collection of large volumes of high-dimensional clinical data \cite{10.1371/journal.pbio.2001402}. A contemporany 
example
is the case of mobile phone data, which generates hundreds of observations per second related to body acceleration, allowing for quasi-continuous monitoring of an individual's physical activity patterns. However, despite opportunities to generate new clinical knowledge, data analytics for precision and digital health must account to new data structures \cite{kosorok2015adaptive, kosorok2019precision, tsiatis2019dynamic, matabuenacontributions}.
 A key challenge in this context is conducting reliable variable selection to enhance interpretability, assess clinical relevance, and safeguard analytical models against the curse of dimensionality \cite{Rahnenfuehrer2023}.

Traditionally, variable selection for multivariate response outcomes has focused on linear models within Euclidean spaces, primarily targeting conditional mean estimation. However, this approach diverges from the realities of modern healthcare data. Complex patient representations—such as functional data and graph-based representations—are increasingly common to capture higher-level 
abstractions of clinical patterns.

For example, continuous glucose monitoring (CGM) devices record glucose values quasi-continuously over time, enabling the derivation of multiple functional and non-functional summary measures (biomarkers) of glucose metabolism \cite{matabuena2024glucodensity}. These biomarkers are especially useful when patient data are collected in free-living conditions, where direct time-series analysis may not be advisable. As a result, the resulting biomarkers often take diverse structures and may be viewed as statistical objects in general metric spaces \((\Omega,d)\)\cite{matabuenacontributions}, such as the space of probability distributions.

\begin{figure}[H] \centering \includegraphics[width=0.8\textwidth]{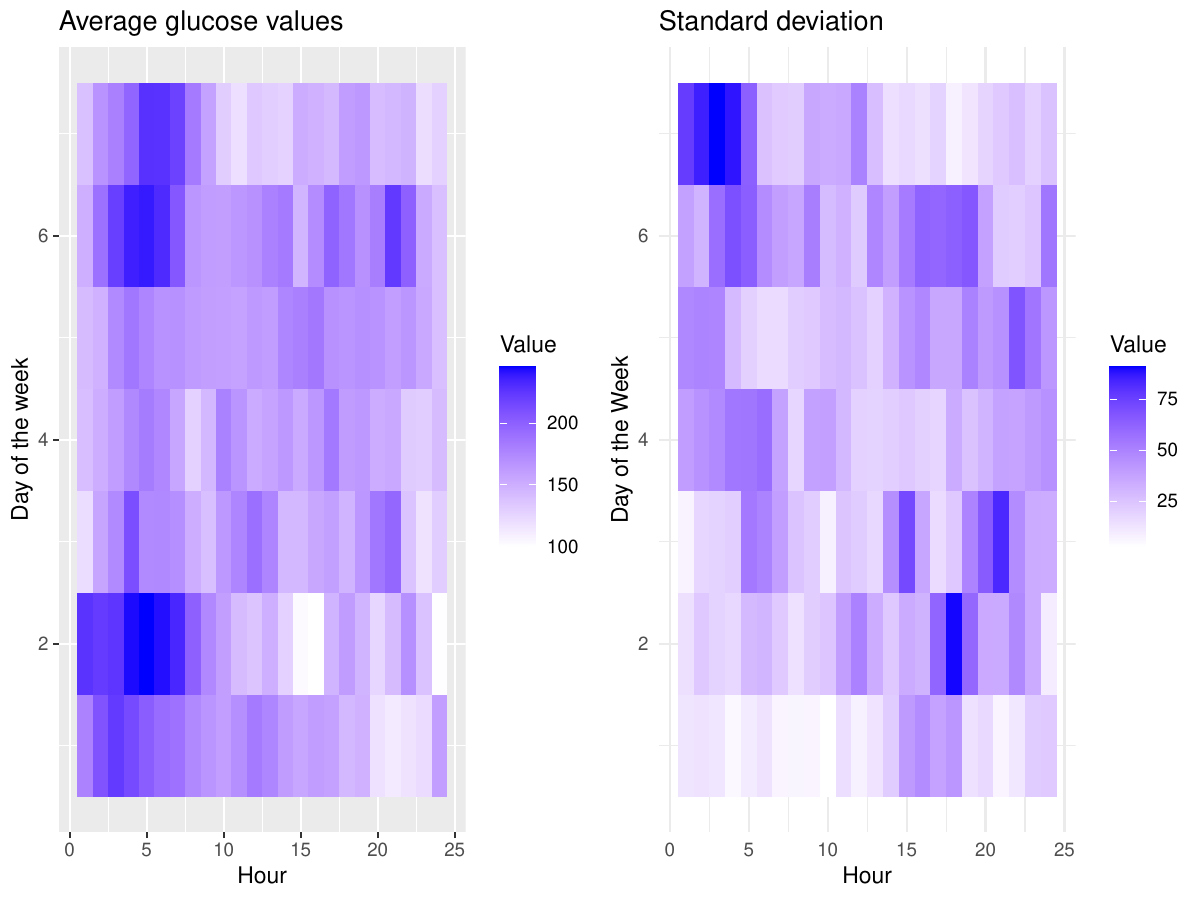} 
\caption{Variation in the mean and standard deviation of glucose values for a diabetic individual depending on the day of the week and time of day.} 
\label{fig:new}
 \end{figure}

Figure \ref{fig:new} illustrates how the mean and standard deviation of glucose concentration for a diabetic individual over a one-month period, across the day and beetween day of the week. The temporal pattern highlights the importance of examining changes at different temporal resolutions and in various time series characteristics. Therefore, consider temporal probability distributions as clinical outcomes captures more information about time series distributional patterns,  than focusing on simple summary statistics like the mean and standard deviation.

 Another critical challenge lies in aggregating and integrating heterogeneous information from various sources, such as genomics, wearables, and electronic health records. These new emerging multimodal structures  motivates the development of new models, in general situations that go beyond linear statistical methods and Euclidean data structures.

Given the absence of general variable selection methods for statistical responses defined in abstract spaces, and the limitations of existing methods primarily designed for linear regressions and Euclidean predictors, this paper proposes a novel statistical and computational algorithm for variable selection in regression models. To the best of our knowledge, the only existing general variable selection method for similar modeling purposes is  the Fréchet Ridge Selection Operator (\texttt{FRISO}) \cite{tu2020era}. Unlike \texttt{FRISO}, which relies on computationally intensive matrix computations and optimization steps, our algorithm employs a novel subgradient projection technique that significantly reduces computational time, allowing for application to large-scale datasets. \texttt{FRISO} is only suitable for datasets with a few hundred observations and a moderate number of predictors ($p$). With these limitations, these methods lack the practical utility necessary to derive new clinical insights in large-scale studies, such as the UK Biobank and All of Us Research Program.

The primary aim of this paper is to demonstrate the effectiveness and scalability of our novel variable selection framework, which is capable of handling datasets comprising millions of patients in biomedical applications. We ilustrate the practical utility of the model on clinical cases involving continuous patient monitoring through digital technologies. From a technical standpoint, we examine spaces of strong negative type, which can be embedded in a Hilbert space with  linear structure. Such spaces are especially relevant in modern healthcare, as they accommodate  random objects used in the clinical applications, such as Laplacian graphs or probability distributions within the $2$-Wasserstein space \cite{matabuenacontributions}. 


Our variable selection framework is highly flexible, capable of integrating diverse information sources and accommodating a variety of models, including additive models and support vector machines. For example, our framework can seamlessly incorporate genomic data, wearable sensor data, and electronic health records to provide a holistic view of patient health. This adaptability is achieved by embedding various loss functions within a general variable selection formulation and applying penalties to address group variable constraints, supporting complex structures such as generalized additive models, distributional models \cite{rigby2005generalized} and functional-to-functional regression models.  Our framework has the potential to greatly advance variable selection across modern healthcare applications, driving forward precision medicine and public health


\subsection{Literature overview}
There is a large body of  literature on variable selection methods for univariate responses \cite{buhlmann2011statistics, bertsimas2016best, huang2012selective, zhang2012general, buhlmann2014high}. Many of these methods were created by identifying biological factors associated with developing and evolving diseases such as Single Nucleotide Polymorphisms  (SNPs) or genetic expressions \cite{buhlmann2014high, hastie2004efficient, he2011variable}. 
The least absolute shrinkage and selection operator (Lasso) \cite{tibshirani1996regression}, which consists of a linear model regularized via an  $\ell_1$ penalty in the objective, is admittedly the most popular method in the literature and in practice. However, its performance could be improved in settings with highly correlated predictors, where it acts more as a screening method than a variable selection one. Motivated by these shortcomings, several extensions or modifications have been proposed in the literature, such as the combination of the $\ell_1$ norm with the $\ell_2$ norm (elastic net) \cite{zou2005regularization}, non-convex penalties \cite{fan2001variable} (MCP and SCAD), efficient methods in the $\ell_0$ norm that from the theoretical and empirical point of view have properties that overcome the limitations mentioned above \cite{bertsimas2016best, hazimeh2020fast}. Significant efforts have been made to modify these algorithms to detect genetic structures with control of the false discovery rate \cite{benjamini1995controlling,su2017false, bogdan2015slope} in different settings, especially in genetic association studies \cite{brzyski2017controlling, sesia2019gene, sesia2020multi}. In addition, to overcome computational challenges, several greedy algorithms were proposed during this period  \cite{efron2004least, xiao2015convexlar, khan2007robust}. 

In a recent series of work for linear structures \cite{bertsimas2020sparse,bertsimas2021sparse}, we have proposed an efficient algorithm for regression problems with the $\ell_0$ and $\ell_2$ penalty that obtains good empirical performance and better phase transition theoretical properties than existing methods \cite{bertsimas2021sparse}. Moreover, their method  applies to arbitrary convex loss function $\ell(\cdot, \cdot )$ in both classification and regression problems \cite{bertsimas2020sparse,bertsimas2021sparse}.  
\paragraph*{Statistical modeling in general spaces}
\noindent	One of the most prominent applications of statistical modeling in metric spaces is in biomedical problems \cite{rodriguez2022contributions}. In personalized and digital medicine applications, it is increasingly common to measure patients' health conditions with complex statistical objects, such as curves and graphs, which allow recording patients' physiological functions and measuring the topological connectivity patterns of the brain at a high resolution level. For example, in recent work, the concept of "glucodensity"  \cite{doi:10.1177/0962280221998064} has been coined, which is a distributional representation of a patient's glucose profile that improves existing methodology in diabetes research \cite{matabuena2022kernel}. This representation is also helpful in obtaining better results with accelerometer data \cite{matabuena2021distributional,ghosal2021scalar,matabuena2022physical, ghosal2023multivariate}.
	
	\noindent From a methodological point of view, statistical regression analysis of response data in metric spaces is a novel research direction \cite{fan2021conditional,chen2021wasserstein,petersen2021wasserstein, zhou2021dynamic, dubey2022modeling, 10.3150/21-BEJ1410, 10.3150/21-BEJ1410, kurisumodel,chen2023sliced}. The first papers on hypothesis testing \cite{10.1214/20-AOP1504,dubey2019frechet,petersen2021wasserstein,fout2023fr}, variable selection \cite{tucker2021variable}, missing data \cite{matabuena2024personalized}, 
multilevel models \cite{matabuena2024multilevel,
bhattacharjee2023geodesic}, uncertainity quantification \cite{matabuena2024conformal,lugosi2024uncertainty} ,dimension-reduction \cite{zhang2022nonlinear},  semi-parametric regression models \cite{bhattacharjee2021single, ghosal2023predicting}, semi-supervised algorithms \cite{	qiu2024semisupervised}
 and non-parametric regression models \cite{schotz2021frechet,hanneke2022universally, bulte2023medoid, bhattacharjee2023nonlinear} have recently appeared.

\subsection{Applications}

The landscape of medical data collection has been dramatically reshaped by technological advancements, as evidenced by the increasing prevalence of high-frequency time-series data in medicine \cite{bunn2018current, knight2021mobile, dunn2018wearables, karczewski2018integrative, wu2022network}. A notable application of this technological evolution is in omics data technologies \cite{karczewski2018integrative, wu2022network}. Omics data encompasses high-throughput methods for analyzing various biological molecules—genes, proteins, and metabolites—to elucidate their roles and interactions in biological processes. The resulting data sets are vast and complex, necessitating computational techniques to uncover valuable insights into disease biology. Furthermore, the advent of biosensor technology has revolutionized the continuous measurement of physiological, biomechanical, and environmental variables \cite{bunn2018current, knight2021mobile, dunn2018wearables}. These devices, often wearable, are capable of monitoring an array of conditions, such as heart rate, temperature, and air quality, thereby enhancing the precision of predictive models in both precision and digital medicine \cite{kosorok2019precision}. Precision medicine tailors medical treatment to individual genetic, environmental, and lifestyle factors, while digital medicine employs digital technologies for disease monitoring, diagnosis, and treatment.

To fully leverage the breadth of data provided by these technologies, random objects are often defined within a metric space. These emerging data structures can include functional and distributional glucose profiles \cite{matabuena2021glucodensities, matabuena2021distributional}, brain connectivity graphs \cite{dubey19, dubeyrssb}, or multivariate data that simultaneously assess several correlated patient characteristics, such as in medical imaging or blood diagnostic tests \cite{angelopoulos2022image}.

\subsection{Specific Scientific Goals}
In medical data analysis, the selection of relevant variables is of paramount importance for creating precise, individualized models for disease diagnosis, predicting patient prognosis, and dynamically prescribing treatments \cite{tsiatis2019dynamic}. The complexities of medical data, which include high dimensionality, heterogeneity, and sparsity, pose significant challenges for statistical modeling. To address these challenges, we propose a novel variable selection methodology based on the sparse optimization framework developed by \cite{bertsimas2020sparse}. Our approach adapts this framework to complex statistical objects, which are increasingly prevalent in medical research \cite{matabuenacontributions}. We demonstrate the efficacy of our methodology through a detailed analysis of diverse information sources, as shown in Table 1. This analysis also highlights the potential improvement in health outcomes when considering complex random objects as a response variable, offering new insights into human health.

From a literature perspective, our framework builds on the well-established approach for variable selection in univariate responses, as proposed by Bertsimas et al. (2021) \cite{bertsimas2021sparse}, which combines $\ell_0$ and $\ell_2$ norm regularizations and shows promising results in computational and statistical efficiency compared to existing methods.

\begin{table}[ht!]
\centering
\renewcommand{\arraystretch}{1.5} 
\begin{tabularx}{\textwidth}{|X|X|X|}
\hline
\textbf{Source of Information} & \textbf{Dataset Description} & \textbf{Clinical Relevance} \\
\hline
\emph{Euclidean Multivariate Data} & Large cohort with over 58,000 patients and $27$ clinical predictors. All predictors are scalar and categorical variables. & Identify the most relevant factor influencing two diagnostic biomarkers of diabetes in a large patient cohort. \\

\hline
\emph{Bivariate Longitudinal Data} & Over 350 patients with 12 clinical predictors and six time points in the response. All predictors are scalar and categorical variables. & Identify the most relevant factor impacting the longitudinal evolution of glucose mean values and variability over time. \\
\hline
\emph{Probability Distribution} & Over 350 patients with 12 clinical predictors, represented over a quantile function defined in a grid of 300 percentiles. All predictors are scalar and categorical variables. & Identify factors influencing the long-term evolution of distributional glucose profiles, capturing key distributional aspects of glucose metabolism. \\
\hline
\emph{Multivariate Distributional Representation as a Response and Histogram Data as a Predictor} & Same dataset as in the previous example, with information from a CGM device over a 24-week (6-month) period. & Detect the most important parts of the histogram (relevant bins) to predict the temporal evolution of probability distribution in a long-term monitoring study within the context of a diabetes clinical trial. \\
\hline

\emph{Combination of Multiple Sources of Information} & Large cohort with over 5,000 patients and 35 clinical predictors. All predictors are scalar and categorical variables. & Predict an individual profile, simultaneously considering physical activity patterns, diabetes status, and the patient's diet. \\
\hline
\emph{Laplacian Graph} & Exact dataset as the previous example. & Identify the most relevant factor impacting a Laplacian graph constructed individually to measure inter-variability patterns. \\
\hline
\end{tabularx}
\caption{Examples of Different Sources of Information Used in This Paper Along with Their Applications}
\label{table:examples0}
\end{table}

\subsection{Paper contributions}
Our contributions can be summarized as follows:

\begin{enumerate}
    \item We propose a novel generalization of the variable selection strategy presented in \cite{bertsimas2020sparse}, tailored to handle complex statistical objects in healthcare applications. Our methods applies to multivariate data, with the ability to fit different regression base algorithms in each random response considered. This extension provides flexibility in selecting relevant variables and is a crucial step towards integrating various sources of medical information at both patient and hospital levels.

    \item We formally prove that our novel variable selection strategies generalize those recently proposed for metric space responses in the Global Fréchet regression model \cite{doi:10.1080/01621459.2021.1969240}. 

    \item We establish statistical consistency results for the new variable selection. 
    

    \item We extend the model formulation proposed from \cite{bertsimas2020sparse} to handle block constraints in categorical and functional predictors, addressing a common issue in medical data not widely examinated for multivariate responses.

    \item We demonstrate the advantages of our proposed methodology over existing state-of-the-art algorithms in various scenarios, showcasing superior model performance and computational feasibility. Our \texttt{Julia} implementation, for example, outperforms existing methods in several orders of magnitude as in the case 
on multivariate Euclidean data and multivariate functions in problems that involve  millons of individuals.
\end{enumerate}

\subsection{Paper outline}

The structure of the paper is outlined below. Section \ref{sec:modelobasico} introduces the general formulation of the new variable selection framework. Section \ref{sec:res} focuses on various case studies, formulates models for each case, and illustrates them with a set of real data presented in Table~\ref{table:examples0}. Section \ref{sec:sim} provides a simulation study to validate the statistical and computational efficiency of our variable selection framework. Finally, Section \ref{sec:discus} discusses the results, models, and extensions of the methodology introduced here. In the Appendix, we introduce several model extensions and extra computation results to support the  advantagues of our proposal.

\section{Model formulation}
\label{sec:modelobasico}

Consider a random sample that is independent and identically distributed (i.i.d.) consisting of $n$ observations,
\[
\mathcal{D}_{n} = \{(X_i, Y_i) \mid i \in [n]\},
\]

\noindent where each pair $(X_i, Y_i) \in \mathcal{X} \times \mathcal{Y} = (\mathcal{X}_{1} \times \dots \times \mathcal{X}_{p}) \times (\mathcal{Y}_{1} \times \dots \times \mathcal{Y}_{m})$ is sampled independently from the same joint distribution $F$. For any integer $n$, we denote $[n] := \{1, \dots, n\}$. In our setting,  $p$ denotes the number of predictors and $m$ the dimension of random responses $Y$. 

For ease of exposition, we assume that $\mathcal{X} = \mathbb{R}^{p}$ and $\mathcal{Y} = \mathbb{R}^{m}$. However, in practice, our algorithm apply when $\mathcal{X}$ and $\mathcal{Y}$ are arbitrary separable Hilbert spaces. In this general case, a large class of linear and semiparametric regression models for a point $x \in \mathcal{X}$, denoted by $m(x) \in \mathcal{Y}$, can be rewritten as a linear combination of coefficients that depend on each individual predictor:
\[
m(x) = \sum_{k=1}^{K} \sum_{j=1}^{p} c_{ij} \phi_{ij}(x),
\]
\noindent where $\{c_{kj}\}_{k=1,j=1}^{K,p}$ are the regression model coefficients, and $\{\phi_{ij}\}_{i=1,j=1}^{K,p}$ are basis functions, as in the trivial case of linear models, additive generalized linear models, or functional regression models. In the latter case, for numerical and practical purposes, altought the prediction and responses are  infinite-dimensional nature, we obtain a high-level approximation by truncating the series to $K$ terms. This allows us to apply a finite-dimensional representation of the regression model parameters, denoted in the subsequent sections as $\beta$, in the variable selection process. 

Our practical goal is to estimate the vector parameter $\beta \in \mathbb{R}^{p \times m}$ of a regression function
\[
m(X, \beta) = \left(m_{1}(X, \beta_{\cdot 1}), \dots, m_{m}(X, \beta_{\cdot m})\right) \in \mathbb{R}^{m}.
\]
\noindent Here, we assume that the $t$-th column of $\beta$, denoted $\beta_{\cdot t}$, parameterizes the model for the $t$-th coordinate of $Y$.

To jointly estimate the $\beta$ parameter from the observations $\mathcal{D}_{n}$ and select the relevant variables from the input space $\mathcal{X}$, we consider the following empirical risk minimization problem:
\begin{equation} \label{eqn:generic.ss}
\begin{aligned} 
\min_{\substack{\beta \in \mathbb{R}^{p\times m} \\ s \in \{0,1\}^{p}}} \quad & \sum_{t=1}^{m} \sum_{i=1}^{n} \ell_{t}\left(Y_{it},\, m_{t}\left(X_{i}, \beta_{\cdot t}\right)\right) + \frac{1}{2\gamma} \sum_{t=1}^{m} \|\beta_{\cdot t}\|_{2}^{2} \\
\text{subject to} \quad & \sum_{j=1}^{p} s_{j} \leq k, \\
& \beta_{jt} = 0 \quad \text{if } s_j = 0, \quad \forall j \in [p], \, t \in [m].
\end{aligned}
\end{equation}

\noindent Here, we associate the prediction of the $t$-th coordinate of $Y$ with an additive convex loss function $\ell_t(\cdot,\cdot)$ ---see Table~\ref{tab:loss.example} for some examples--- and introduce a binary decision variable $s \in \{0,1\}^p$ to encode the subset of variables selected. The linear constraint $\sum_{j=1}^{p} s_{j} \leq k$ ensures that we select no more than $k$ variables out of $p$. The logical constraints ensure that all coefficients associated with variable $j$ ($\beta_{jt}$ for $t\in [m]$) are set to zero when this feature is not selected ($s_j=0$).

The parameter $\gamma \geq 0$ controls the strength of the ridge regularization. Indeed, earlier work \cite{hastie2017extended} demonstrates that having an explicit constraint on the number of variables only (best subset selection formulation) performs poorly when the signal-to-noise ratio (SNR) is low, but incorporating an additional ridge regularization can effectively increase feature selection and predictive accuracy in these settings \cite{bertsimas2020sparse}.
\begin{remark}
By construction, for each $t \in [m]$, we can employ a distinct loss function. In a general formulation of our algorithm, we consider $\mathbb{H}_{Y} = \mathbb{H}_{Y_{1}} \times \dots \times \mathbb{H}_{Y_{m}}$, which represents the Cartesian product of $m$ Hilbert Spaces. Each dimension has the flexibility to be a different Hilbert space, allowing for rich and diverse modeling possibilities for different sources of information such as genetic, wereable, and electronic records data.
\end{remark}

\begin{table}[h]
    \centering
    \begin{tabular}{lll}
        \toprule
        Loss Function & $\ell(y,u)$ & $\hat{\ell}(y,a) = \max_{u} \: ua - \ell(u,a)$ \\
        \midrule
        Ordinary Least Squares (OLS) & $\tfrac{1}{2} (y - u)^2$ & $y a + \tfrac{1}{2} a^2$ \\[1ex]
        Pinball Loss & $\max\{q (y - u),\ (1 - q)(u - y)\}$ & $\begin{cases} 
            y a & \text{if } -q \leq a \leq 1 - q, \\ 
            +\infty & \text{otherwise},
        \end{cases}$ \\
        Logistic Loss & $\log \left(1 + e^{-y u}\right)$ & $-H(-y\alpha)$, for $y\alpha \in [-1, 0]$ \\
        \midrule
    \end{tabular}
    \caption{Examples of loss functions and their Fenchel conjugates. Here, $H(x) = -x\log x - (1 - x)\log(1 - x)$.}
    \label{tab:loss.example}
\end{table}

To solve the optimization problem \eqref{eqn:generic.ss} efficiently, we rely on a saddle-point formulation, which can be obtained by dualizing the minimization problem with respect to $\beta$ for any subset of predictors $s$.

\begin{theorem} \label{thm:saddle.reformulation}
For any convex loss
functions $\ell_t$, and assume additive linear structure across different loss function $\ell_t, t\in [m],$ the optimization problem \eqref{eqn:generic.ss} is equivalent to  
\begin{align}\label{eqn:boundary}   
 \min_{s\in \{0,1\}^{p}: \sum_j s_j \leq k} \quad \max_{\alpha\in \mathbb{R}^{n \times m}} \quad f\left(\alpha,s\right):= 
\left(-\sum_{t=1}^{m}\sum_{i=1}^{n}\hat{\ell}_t\left(Y_{it},\alpha_{it}\right)-\dfrac{\gamma}{2} \sum_{t=1}^{m}\sum_{j=1}^{p}s_{j}\alpha^{\top}_{{\cdot}t}X_{j}X^{\top}_{j}\alpha_{{\cdot}t}\right),
 \end{align}   
\noindent where $\hat{\ell}(y,a):= \max_{u\in \mathbb{R}} u a-\ell(y,u)$ is a convex function known as the Fenchel conjugate of $\ell$ \cite{bauschke2012fenchel}.In particular, the function $f$ is continuous, linear in $s$, and concave in $\alpha$.   
\end{theorem}

In the special case of ordinary least square, i.e., $\ell_t(y,u) = \dfrac{1}{2}(y-u)^2$, the function $f$ is a quadratic function in $\alpha$,
so the inner maximization problem (with respect to $\alpha$) can be solved in closed form. Specifically, for each $t=1,\dots,m$, the maximum is attained at $\alpha^{\star}_{\cdot t}(s) = -(I_{n}+\gamma X_{s}X_{s}^{\top})^{-1} Y_{t}$, 
where $X_{s}X_{s}^{\top}$ concisely denote $\sum_{j=1}^{p} s_{j} X_{j}X_{j}^{\top}$, and the objective value is
\begin{align*}
\frac{1}{2}\sum_{t=1}^{m}  Y_{t}^{\top}(I_{n}+\gamma X_{s}X^{\top}_{s})^{-1}Y_{t}.
\end{align*}

\subsection{Boolean relaxation and efficient algorithms}

As often in discrete optimization, it is natural to consider the Boolean relaxation of problem \eqref{eqn:boundary}
\begin{equation}
\label{eqn:relax}
\min_{s\in [ 0, 1 ]^p \: : \: \textbf{e}^\top   s \leqslant k} \max_{\alpha \in \mathbb{R}^{n\times m}} f(\alpha, s),
\end{equation}
and study its tightness, as done by \cite{pilanci2015sparse}. 

\subsubsection{Tightness result}
The above problem is recognized as a convex/concave saddle point problem. According to Sion's minimax theorem \citep{sion1958general}, the minimization and maximization in \eqref{eqn:relax} can be interchanged. Hence, saddle point solutions $(\bar \alpha, \bar s)$ of \eqref{eqn:relax} should satisfy 
$$\bar \alpha \in \arg \max_{\alpha \in \mathbb{R}^{n\times m}} f(\alpha, \bar s), ~~~
\bar s  \in \arg \min_{s \in [0,1]^p} \,f(\bar \alpha, s) \mbox{ s.t. } ~s^\top   \textbf{e} \leqslant k .$$
Since $f$ is a linear function of $s$, a minimizer of $f(\bar \alpha, s)$ can be constructed easily by selecting the $k$ smallest components of the vector $(-\tfrac{\gamma}{2} \bar \alpha^\top   X_j X_j^\top   \bar \alpha)_{j=1,...,p}$. If those $k$ smallest components are unique, the so constructed binary vector must be equal to $\bar s$ and hence the relaxation \eqref{eqn:relax} is tight. In fact, the previous condition is necessary and sufficient as proven by \cite{pilanci2015sparse}:
\begin{thm}{\citep[Proposition 1]{pilanci2015sparse}} \ref{eqn:boundary}
The Boolean relaxation \eqref{eqn:relax} is tight if and only if there exists a saddle point $(\bar \alpha, \bar s)$ such that the vector $\bar \beta := (\bar \alpha^\top   X_j X_j^\top   \bar \alpha)_{j=1,...,p}$ has unambiguously defined $k$ largest components, i.e., there exists $\lambda \in \mathbb{R}$ such that $\bar \beta_{[1]} \geqslant \cdots \geqslant \bar \beta_{[k]} > \lambda > \bar \beta_{[k+1]} \geqslant \cdots \geqslant \bar \beta_{[p]}$.
\end{thm}
This uniqueness condition in Theorem \ref{eqn:boundary} is frequently met in real-world scenarios. For example, it is satisfied with high probability when the covariates $X_j$ are independent, as shown in \citep[Theorem~2]{pilanci2015sparse}. In other words, randomness simplifies the complexity of the problem. These findings have had a significant practical impact, driving the development of convex proxy-based heuristics like Lasso. As a result, efficient algorithms can be devised to solve the saddle point problem \eqref{eqn:relax} without the need for complex discrete optimization methods.
\subsubsection{Dual sub-gradient algorithm}
In this section, we propose and describe an algorithm for solving problem \eqref{eqn:relax} efficiently. Our algorithm implemented in \verb|Julia| is fast and scales to data sets with $n,p$ in the $100,000$s in few minutes.  

For a given $s$, maximizing $f$ over $\alpha$ cannot be done analytically, with the noteworthy exception of ordinary least squares, whereas minimizing over $s$ for a fixed $\alpha$ reduces to sorting the components of $(-\alpha^\top   X_j X_j^\top   \alpha)_{j=1,...,p}$ and selecting the $k$ smallest. We take advantage of this asymmetry by proposing a dual projected sub-gradient algorithm with constant step-size, as described in pseudo-code in Algorithm \ref{subgradient}. $\delta$ denotes the step size in the gradient update and $\mathcal{P}$ the projection operator over the domain of $f$. { At each iteration, the algorithm updates the support $s$ by minimizing $f(\alpha, s)$ with respect to $s$, $\alpha$ being fixed. Then, the variable $\alpha$ is updated by performing one step of projected sub-gradient ascent with constant step size $\delta$.} The denomination "sub-gradient" comes from the fact that at each iteration $\nabla_\alpha f(\alpha^T, s^T)$ is a sub-gradient to the function $\alpha \mapsto \min_s f(\alpha, s)$ at $\alpha = \alpha^T$. 
 
{ In terms of computational cost, updating $\alpha$ requires $O\left(n \|s \|_0 \right)$ operations for computing the sub-gradient plus at most $O\left(n \right)$ operations for the projection on the feasible domain. The most time-consuming step in Algorithm \ref{subgradient} is updating $s$ which requires on average $O\left(n p + p \log p \right)$ operations.}
 
\begin{algorithm*}
\caption{Dual sub-gradient algorithm}
\label{subgradient}
\begin{algorithmic}
\State $s^0, \alpha^0 \leftarrow$ Initial solution
\State $T = 0$
\Repeat
    \State $s^{T+1} \in \text{argmin}_{s} f(\alpha^{T}, s)$
    \State $\alpha^{T+1} = \mathcal{P} \left( \alpha^T + \delta \nabla_\alpha f(\alpha^T, s^T) \right)$
    \State $T = T + 1$
\Until{Stop criterion}
\State $\hat \alpha_T = \frac{1}{T} \sum_t \alpha^t$
\State $\hat s = \text{argmin}_{s} f(\hat \alpha_T, s)$
\end{algorithmic}
\end{algorithm*}

\subsection{Extension to Group Variable Selection}

In some specific use cases, variable selection needs to respect certain block constraints between variables, e.g., when using one-hot encoding of categorical variables with more than two levels or in the context of nonlinear additive models. In this section, we describe how our model can be adapted to account for such requirements.

Let us assume that the set of input variables $[p]$ is partitioned into $q$ disjoint groups, $\mathcal{S}_{1}, \dots, \mathcal{S}_{q}$. In this context, the binary variable $s$ is now indexed by $u \in [q]$ and indicates whether group $u$ is selected. The optimization problem for the joint group variable selection and model estimation can be written as
\begin{align}\label{eqn:group}
\min_{\substack{\beta \in \mathbb{R}^{p\times m} \\ s \in \{0,1\}^{q}}} \quad & \sum_{t=1}^{m} \sum_{i=1}^{n} \ell_{t}\left(Y_{it},\, m_{t}\left(X_{i}, \beta_{\cdot t}\right)\right) + \frac{1}{2\gamma} \sum_{t=1}^{m} \|\beta_{\cdot t}\|_{2}^{2} \\
\text{subject to} \quad & \sum_{u=1}^{q} s_{u} \leq k, \\
& \beta_{jt} = 0 \quad \text{if } s_u = 0, \quad \forall u \in [q], \, j \in \mathcal{S}_u, \, t \in [m].
\end{align}
This is equivalent to \eqref{eqn:generic.ss} in the case when $q = p$ and $\mathcal{S}_u = \{ u \}$ for all $u \in [p]$.

\begin{remark}
Estimating the conditional mean operator $\mathbb{E}(Y \mid X) = f(X)$ without assuming specific functional shape constraints is challenging due to the curse of dimensionality. One effective strategy to address this issue is to constrain the functional form of the operator to be additive, though not necessarily linear, as $f(X) = \sum_{j=1}^{p} f_{j}(X^{j})$, where each $f_j(\cdot)$ denotes a non-linear regression function for each covariate. A common approach is to select $f_{j}$ for all $j \in \{1,\dots,p \}$ from a smooth functional class of splines $\mathcal{C}_{j}(\mathcal{R})$, i.e., $\mathcal{C}_j = \{ f \in  L^{2}(\mathcal{R}) \mid f \text{ is twice continuously differentiable over } \mathcal{R} \}$, where $\mathcal{R}$ is the domain defining the functional space. Given observations $(X_{i}, Y_{i})$, $i \in [n]$, the additive model $f(X)$ can be estimated by solving the following optimization problem:
\begin{equation}
\min_{f \in \mathcal{C}_{j}} \sum_{i=1}^{n} \left( Y_{i} - \sum_{j=1}^{p} f_{j}(X^{j}_{i}) \right)^2 + \lambda \sum_{j=1}^{p} \operatorname{Pen}(f_{j}),
\end{equation}
\noindent where $\operatorname{Pen}(f_{j})$ is a roughness penalty that ensures the smoothness of the function $f_{j}$. Given the linear representation of splines basis, this approach aligns with our model formulation, as spline-based optimization can be reinterpreted as a linear optimization problem, matching the group lasso formulation.
\end{remark}

From a practical perspective, to address the group $\ell_{0}$ constraints introduced in model \eqref{eqn:generic.ss}, and assuming that $[p]$ is partitioned into $q$ disjoint groups, $\mathcal{S}_{1}, \dots, \mathcal{S}_{q}$, we can leverage the original implementation of the dual sub-gradient algorithm from model \eqref{eqn:generic.ss}. This implementation is specifically designed to handle group constraints for multivariate response restrictions on the $\boldsymbol{\beta}$ coefficients, in the same spirit as the group model \eqref{eqn:group}. However, in our specific scenario, the decision unit is the group index $\mathcal{S}_{j}$, where $j \in [q]$, rather than the individual variable index $j \in [p]$. From a computational perspective, we only need to map the specific group index in the original implementation to the corresponding index of  variables.

\subsection{Variable Selection in Metric Space Responses}
In this section, we introduce a novel mathematical connection between our variable selection framework and existing methods for variable selection in linear regression models on metric spaces \cite{petersen2019frechet}, specifically the individual ridge regression algorithm introduced in \cite{tucker2021variable}. Technically, our method is more comprehensive; it not only encompasses the aforementioned algorithm as a special case but also offers advantages in terms of computational efficiency. Furthermore, our approach can be seamlessly integrated with other techniques, such as stability selection, which helps in detecting the level of sparsity in the variable selection problem. Before delving into the specifics of our methodology, we provide some essential mathematical background on regression modeling in metric spaces.

Let $(X,Y)\in \mathcal{X} \times \mathcal{Y}$ be a pair of random variables that play the role of the predictor and response variable in a regression model.  We assume that	$\mathcal{X}= \mathbb{R}^{p}$ and $\mathcal{Y}$ is a  separable metric space equipped with a distance $d_1$.  We assume that there exists $y\in \mathcal{Y}$ such that $\mathbb{E}(d^{2}(Y,y)) < \infty$. The regression function $m$ is defined as
	
	\begin{equation}
	m(x)= \argminB_{y\in \mathcal{Y}} \mathbb{E}(d^{2}(Y,y)|X=x),
	\end{equation}
	
	\noindent	where $x\in \mathbb{R}^{p}$.  In other words, $m$ is the conditional Fréchet mean \cite{petersen2019frechet}. For simplicity, we assume that the minimum in $(1)$ is achieved for each $x$ and moreover, the conditional Fréchet mean of $Y$ given $X=x$, is unique. We note here that one may also consider the conditional Fréchet median obtained by $\argminB_{y\in \mathcal{Y}}\mathbb{E}(d(Y,y)|X=x).$ However, this is a special case of our setup obtained by replacing the metric $d$, by $\sqrt{d}$ (which is also a metric).

\subsubsection{Global Fréchet Model}
Recent advances in object-oriented data analysis \cite{marron2021object}, particularly in handling complex statistical objects in non-standard spaces like graphs and probability distributions, have significant implications in medical research \cite{marron2014overview, huckemann2021data}. Here, we focus on the global Fréchet regression model, an extension of the linear regression model for modeling responses in separable metric spaces \cite{petersen2019frechet}.

 Let $(X,Y) \in \mathcal{X} \times \mathcal{Y}$ be a bivariate random variable, where $X \in \mathcal{X} = \mathbb{R}^{p}$, and $Y \in \mathcal{Y}$ is a separable metric space. For a new point $x \in \mathbb{R}^{p}$, we assume that the functional form of Frechet mean $m(\cdot)$  is obtained  by solving the following optimization problem in metric space:
\begin{equation*}
	m(x) =  \argminB_{y \in \mathcal{Y}} \mathbb{E}\left[\omega(x,X) d^{2}(Y,y) \right],
\end{equation*}
\noindent where $\Sigma = \operatorname{Cov}(X,X)$, $\mu = \mathbb{E}(X)$, $\omega(x,X) = \left[1 + (X-\mu)\Sigma^{-1}(x-\mu)\right]$.

To estimate the conditional mean function $m(x)$ from a random sample $\mathcal{D}_n = \{(X_i, Y_i)\}_{i=1}^n$, we solve the empirical counterpart problem:
\begin{equation}
	\widehat{m}(x) = \argminB_{y \in \mathcal{Y}} \frac{1}{n} \sum_{i=1}^{n} [\omega_{in}(x) d^{2}(y, Y_i)],
\end{equation}
\noindent where $\omega_{in}(x) = \left[ 1 + (X_{i} - \overline{X})\widetilde{\Sigma}^{-1}(x - \overline{X}) \right]$, with $\overline{X} = \frac{1}{n} \sum_{i=1}^{n} X_i$, and $\widetilde{\Sigma} = \frac{1}{n-1} \sum_{i=1}^{n} (X_i - \overline{X})(X_i - \overline{X})^\intercal$.

\subsubsection{Variable Selection in Metric Spaces with Penalized Ridge Algorithm}

Following \cite{tu2020era},
we now extend our discussion to introduce a variable selection technique for the global Fréchet regression, incorporating the concept of individually penalized ridge Fréchet regression. By introducing $\boldsymbol{v}=\lambda^{-1}$ in conjunction with the earlier formulation, we aim to optimize the following expression:
\begin{equation}\label{eqn:friso}
\widehat{R}_{\oplus}(\mathbf{x} ; \lambda^{-1})=\min_{y \in \mathcal{Y}} \sum_{i=1}^{n} \left\{1+(x-\overline{X})^{\top} \operatorname{diag}(\sqrt{\lambda})\left[\operatorname{diag}(\sqrt{\lambda}) \widetilde{\Sigma} \operatorname{diag}(\sqrt{\lambda})+\mathbf{I}\right]^{-1} \operatorname{diag}(\sqrt{\lambda})(X_{i}-\overline{X})\right\} d^{2}(Y_{i},y),
\end{equation}
\noindent where $\mathbf{I}$ is the identity matrix $p\times p,$ and  the optimizer $\widehat{R}_{\oplus}(\mathbf{x} ; \lambda^{-1})$ corresponds to the solution of this minimization problem, and subject to the constraints $\lambda_{j} \geq 0$ for $j=1,2, \ldots, p$, and $\sum_{j=1}^{p} \lambda_{j}=\tau$ for some $\tau \geq 0$. The solution, denoted by $\widehat{\lambda}(\tau) = \left(\hat{\lambda}_{1}(\tau), \hat{\lambda}_{2}(\tau), \ldots, \hat{\lambda}_{p}(\tau)\right)^{\top}$, allows for variable selection by identifying significant predictors. Specifically, the set of important predictors is estimated as $\widehat{\mathcal{I}}(\tau)=\{j\in  \{1\dots,p\}: \widehat{\lambda}_{j}(\tau)>0\}$, mirroring the approach in linear regression scenarios. This method is referred to as the Fréchet Ridge Selection Operator (\texttt{FRISO}).

The implementation of \texttt{FRISO} necessitates the tuning of the parameter $\tau>0$, which, given sufficient data, is best determined using an independent validation set $ \mathcal{D}_{2}\subset  \{\left(X_{i}, Y_{i}\right): i\in [n]
\}$. The optimal $\tau$ is obtained by minimizing the sum $\sum_{i=1}^{n} d^{2}(Y_{i}, \widehat{R}_{\oplus}(X_{i} ;(\lambda(\tau)^{-1}))$ over $\tau>0$, typically through a grid search method for practical optimization.

\subsubsection{Metric spaces of Negative Type }

We now focus on a general family of metric space statistical objects for which our algorithm exhibits tractable mathematical and computational properties.  First,  consider the separable metric space $(\mathcal{Y},d)$. Denote $M(\mathcal{Y})$ as the set of probability measures on $\mathcal{Y}$, with $M_p(\mathcal{Y})$ representing the probability measures equipped with $p$-th moments:

\begin{equation*}
    M_p(\mathcal{Y}) = \left\{ v \in M(\mathcal{Y}) : \exists \, \omega \in \mathcal{Y}, \int_{\mathcal{Y}} d^p(\omega,x) \, dv(x) < \infty \right\}.
\end{equation*}
This section delves into the intricacies of a metric space $(\mathcal{Y},d)$ of negative type. 

\begin{definition}
A metric space $(\mathcal{Y},d)$ is of negative type if for any finite set $y_1, \ldots, y_n \in \mathcal{Y}$, the following condition holds:
\begin{equation}
    \sum_{i=1}^{n} \sum_{j=1}^{n} f_i f_j d(y_i, y_j) \leq 0,
\end{equation}
for any $f_i \in \mathbb{R}$ with $\sum_{i=1}^n f_i = 0$. Additionally, for distinct measures $v_1, v_2 \in M_1(\mathcal{Y})$, the inequality 
\begin{equation*}
    \int d(x_1, x_2)^2 \, dv_{-}(x_1, x_2) \leq 0
\end{equation*}
is satisfied, where $v_{-} = v_1 - v_2$.
\end{definition}

\begin{theorem}[Schoenberg (1937, 1938)\cite{schoenberg1937certain,schoenberg1938metric}]
    Let $\Omega=(\mathcal{Y},d)$ be of negative type if and only if there is a Hilbert space $\mathcal{H}$ and a map $\phi: \Omega \rightarrow \mathcal{H}$ such that $\forall y, y^{\prime} \in \mathcal{H}, \hspace{0.2cm} d\left(y, y^{\prime}\right)=\left\|\phi(y)-\phi\left(y^{\prime}\right)\right\|^{2}.$ 
\end{theorem}

Now, we consider a couple of examples of metric spaces of negative type without a vector space structure that can be embedded in a separable Hilbert space.

\begin{example} [Laplacian graph]
Consider a graph $G = (V, E)$ with vertex set $V=\{1,\dots,m\}$ and edge set $E$. The Laplacian matrix $L_G$ is defined as:
\begin{equation}
   L_G = D_G - A_G,
\end{equation}
\noindent where $D_G$ and $A_G$ are the degree and adjacency matrices of $G$, respectively. If we consider the graph--space of the form $ G=(V,E)$ equipped with the Euclidean distance $\|\cdot-\cdot\|_{2}$, the resulting space is of negative type.
\end{example}

\begin{example}[The 2-Wasserstein Distance in the univariate case]
Consider the space of distribution functions on $\mathbb{R}$ with finite second moments, denoted by $\mathcal{W}_2(\mathcal{D}(\mathbb{R}))$. This forms a metric space under the 2-Wasserstein metric that is of negative type. Specifically, the 2-Wasserstein distance between two probability measures $\mu$ and $\nu$ is defined as $\mathcal{W}_2(\mu,\nu)$ and given by the following expression:
\begin{equation}
   \mathcal{W}_2^{2}(\mu,\nu) = \left(\inf_{\gamma \in \Pi(\mu,\nu)} \int_{\mathbb{R}^{2}} \|x-y\|^{2} \, d\gamma(x,y) \right)^{1/2} = \left(\int_{0}^{1} (Q_{\mu}(t)-Q_{\nu}(t))^2 \, dt \right)^{1/2},
\end{equation}
\noindent where $\Pi(\mu,\nu)$ denotes the set of all joint probability measures on $\mathbb{R}^2$ with marginals $\mu$ and $\nu$, and $Q_{\mu}$. Here, $Q_{\mu}$ and $Q_{\nu}$ represent the quantile functions (qdf) of $\mu$ and $\nu$, respectively. The expression on the right-hand side is particularly known as the 2-Wasserstein distance between the CDFs of $\mu$ and $\nu$, and the closed-expressions, only is valid for univariate probability distributions.
\end{example}

\begin{remark}
Spherical data equipped with the angular distance is another example of spaces of negative type. This space is very important for analyzing vector-valued compositional data in the presence of zeros, projecting the compositional data onto the sphere, or analyzing directional spherical data itself \cite{lyons2020strong}
\end{remark}

\subsubsection{Global Fréchet Model for Negative Type Spaces}

Consider the scenario where $(\mathcal{Y}, d)$ is a metric space of negative type. In this context, the squared distance can be expressed as:
\[
d^{2}(y, y') = \|\phi(y) - \phi(y')\|^{2} = \langle \phi(y) - \phi(y'), \phi(y) - \phi(y') \rangle_{\mathcal{H}},
\]
\noindent where $\phi: \mathcal{Y} \rightarrow \mathcal{H}$ is an embedding into a Hilbert space $\mathcal{H}$.

In this case, the Global Fréchet model takes the form:
\[
m(x) = \arg\min_{y \in \mathcal{Y}} \mathbb{E}\left[\omega(x, X) \langle \phi(Y) - \phi(y), \phi(Y) - \phi(y) \rangle_{\mathcal{H}} \right],
\]
\noindent where $\omega(x, X)$ is a weighting function. We denote $m^{\phi}$ as the regression model calculated in the transformed space $\phi(Y) \in \mathcal{H}$.

\begin{proposition}
    Suppose that $(\mathcal{Y}, d)$ is of negative type. The Global Fréchet model can be reinterpreted as a standard linear regression model in the form:

\begin{equation}\label{model:infi}
m^{\phi}(x) = \beta_0 + \beta^{\top}(x - \mu), \quad \text{with} \quad \mu = \mathbb{E}(X),
\end{equation}

\noindent where $\beta_0$ denotes the functional intercept and $\beta$ represents the functional slope coefficients in the transformed Hilbert space $\mathcal{H}$. In addition, the set of important variables of $m$ and $m^{\phi}$ are the same.
\end{proposition}

\begin{proof}
    Consequently of \cite{petersen2019frechet} and the characterization of negative type spaces from \cite{schoenberg1937certain,schoenberg1938metric}.
\end{proof}

\begin{remark}
When $\beta \in \mathcal{H}$ belongs to an infinite-dimensional Hilbert space, the model defined in Equation \ref{model:infi} is, for practical calculation purposes, specified in our setting by the general model of Equation \ref{eqn:generic.ss}. This is because, in practice, we approximate $\beta$ using a finite-dimensional basis of functions.
\end{remark}







 




\subsubsection{Mathematical Equivalence Between \texttt{FRISO} and the Proposed Approach}

In this section, we establish the mathematical equivalence between \texttt{FRISO} and our proposed approach by leveraging the concept of perspective functions in low-rank optimization problems, as introduced in \cite{bertsimas2023new}.

Our problem, as defined in the previous section, is:

\begin{equation} \label{eqn:generic_ss}
\begin{aligned} 
\min_{\substack{\beta \in \mathbb{R}^{p \times m}, \\ s \in \{0,1\}^{p}}} \quad & \sum_{t=1}^{m} \sum_{i=1}^{n} \ell_{t}\left(Y_{it},\, m_{t}\left(X_{i}, \beta_{\cdot t}\right)\right) + \frac{1}{2\gamma} \sum_{t=1}^{m} \|\beta_{\cdot t}\|_{2}^{2} \\
\text{subject to} \quad & \sum_{j=1}^{p} s_{j} \leq k, \\
& \beta_{jt} = 0 \quad \text{if } s_j = 0, \quad \forall j \in [p],\ t \in [m].
\end{aligned}
\end{equation}

\noindent Here, $\beta \in \mathbb{R}^{p \times m}$ represents the matrix of regression coefficients, $s \in \{0,1\}^p$ is a binary vector indicating the selection of features, $\ell_{t}$ is the loss function for task $t$, and $m_t$ is the model for task $t$.

We aim to reformulate this problem by incorporating the feature selection variables $s_j$ directly into the optimization, using the notion of perspective functions to handle the regularization term.

\begin{definition}
 The perspective function of a convex function $f: \mathbb{R}^n \rightarrow \mathbb{R}$ is defined as:

\[
g_f(x, t) = 
\begin{cases}
t f\left( \dfrac{x}{t} \right), & \text{if } t > 0, \\
0, & \text{if } t = 0,\, x = 0, \\
+\infty, & \text{otherwise}.
\end{cases}
\]
   
\end{definition}

In our case, we consider the function $f(x) = \dfrac{1}{2\gamma} \|x\|_2^2$, corresponding to the ridge penalty. The perspective function becomes:

\[
g_f(x, t) = 
\begin{cases}
\dfrac{1}{2\gamma t} \|x\|_2^2, & \text{if } t > 0, \\
0, & \text{if } t = 0,\, x = 0, \\
+\infty, & \text{otherwise}.
\end{cases}
\]

Using the perspective function, we rewrite the regularization term in \eqref{eqn:generic_ss} as:

\[
\frac{1}{2\gamma} \sum_{t=1}^{m} \|\beta_{\cdot t}\|_2^2 = \sum_{j=1}^{p} \sum_{t=1}^{m} \frac{1}{2\gamma} \beta_{jt}^2 = \sum_{j=1}^{p} \sum_{t=1}^{m} g_f(\beta_{jt}, s_j).
\]

This formulation effectively incorporates the feature selection variables $s_j$ into the regularization term. The convention $\beta_{jt}^2 / 0 = +\infty$ when $s_j = 0$ enforces $\beta_{jt} = 0$ if $s_j = 0$.

The optimization problem now becomes:

\begin{equation} \label{eqn:reformulated_problem}
\begin{aligned} 
\min_{\substack{\beta \in \mathbb{R}^{p \times m},\\ s \in \{0,1\}^{p}}} \quad & \sum_{t=1}^{m} \sum_{i=1}^{n} \ell_{t}\left(Y_{it},\, m_{t}\left(X_i, \beta_{\cdot t}\right)\right) + \sum_{j=1}^{p} \sum_{t=1}^{m} \frac{1}{2\gamma s_j} \beta_{jt}^2 \\
\text{subject to} \quad & \sum_{j=1}^{p} s_j \leq k.
\end{aligned}
\end{equation}

To handle the division by $s_j$, we introduce auxiliary variables $\rho_{jt} \geq 0$ and rewrite the regularization term:

\[
\sum_{j=1}^{p} \sum_{t=1}^{m} \frac{1}{2\gamma s_j} \beta_{jt}^2 = \frac{1}{2\gamma} \sum_{j=1}^{p} \sum_{t=1}^{m} \rho_{jt},
\]

subject to the constraints:

\[
\beta_{jt}^2 \leq s_j \rho_{jt}, \quad \forall j \in [p],\ t \in [m].
\]

This ensures that when $s_j = 0$, the constraint $\beta_{jt}^2 \leq 0$ forces $\beta_{jt} = 0$.

\begin{proposition}
The optimization problem defined in Equation \eqref{eqn:generic_ss} is equivalent to solve the similar individual \textit{FRISO} ridge regression algorithm introduced in \cite{tucker2021variable}.
\end{proposition}




\subsection{Statistical  Theory}

In this section, we show that with the proper choice of the variable selection method \textit{FRISO} and our equivalent subgradient proposal, we can select the correct variables as \( n \) grows to infinity.

\begin{theorem}[\cite{tucker2021variable}]
    Assume that conditions (U0)--(U2) of Theorem 2 in \cite{petersen2019frechet}. Under Conditions [A--D] (see Appendix), when \( \tau = \tau_{n} \rightarrow \infty \) as \( n \rightarrow \infty \), the solution \( \widehat{\lambda}\left(\tau_{n}\right) \) of Equation~\ref{eqn:friso} satisfies \( \hat{\lambda}_{j}\left(\tau_{n}\right) \xrightarrow{p} \infty \) for \( j \in \mathcal{I} \) and \( \hat{\lambda}_{j'}\left(\tau_{n}\right) \xrightarrow{p} 0 \) for \( j' \notin \mathcal{I} \) as \( n \rightarrow \infty \).
\end{theorem}

\begin{remark}
    Theorem 5 only provides a weak conclusion that \( \hat{\lambda}_{j'}\left(\tau_{n}\right) \xrightarrow{p} 0 \) for \( j' \notin \mathcal{I} \) as \( n \rightarrow \infty \) since the objective function (15) is a highly complex function of \( \lambda \). To obtain a stronger result, such as attaining zero almost surely, would require characterizing the gradient of the objective function, which is very challenging for Fréchet regression due to the general non-Euclidean setting.
\end{remark}

The next result shows that Conditions [B]--[D] are satisfied in the linear model in the transformed space for spaces of negative type. Therefore, the previously established consistency results, where \( n \) and \( p \) are fixed (our setting in massive datasets), hold in our environment.

\begin{proposition}
    Conditions [B]--[D] are satisfied by the linear regression model 
    \[
    \phi(Y) = \beta_{0} + \mathbf{X}^\top \boldsymbol{\beta} + \epsilon,
    \]
    which relates to the regression model in the negative-type space given by
    \begin{equation}\label{model:infi}
    m^{\phi}(x) = \beta_0 + \beta^{\top}(x - \mu), \quad \text{with} \quad \mu = \mathbb{E}(X).
    \end{equation}
\end{proposition}

\subsection{Computational implementations details}
The methodologies presented in this paper were implemented using \texttt{Julia}, a highly efficient computational programming language. Specifically, we expanded the \textit{SubsetSelection} package in \texttt{Julia}, which is designed for variable selection of scalar responses. To enhance the functionality and accessibility of our tools, we integrated our \texttt{Julia} implementations with the \texttt{R} programming environment using the \texttt{Rjulia} package. This integration allows users to seamlessly combine the robust computational capabilities of \texttt{Julia} with the extensive statistical libraries available in \texttt{R} for different regression models.

\section{Applications in Scientific Problems Related to Diabetes Mellitus}\label{sec:res}

\subsection{Diabetes Mellitus Disease as Motivation to Analyze Different Clinical Outcomes Structures}

Diabetes Mellitus (DM) \cite{adeghate2006update} is a complex metabolic disorder characterized by elevated blood glucose levels, resulting from insufficient insulin production or the body's inability to use insulin effectively. Nowadays, in modern societies, DM has become a critical public health problem \cite{caspersen2012aging}, affecting millions of individuals and placing substantial strain on healthcare systems worldwide. DM is associated with a spectrum of complications, including cardiovascular risk, kidney failure, blindness, and lower limb amputations \cite{harding2019global}. The escalating prevalence of diabetes is linked to lifestyle changes, sedentary habits, and poor dietary choices, demanding innovative public health approaches for disease management. New approaches based on the principles of precision public health, particularly from the perspective of preventive medicine, are essential \cite{matabuena2024deep}.

In DM, traditional interventions include lifestyle modifications such as regular physical activity and pharmacological treatments like insulin therapy. For non-diabetic individuals, lifestyle modifications \cite{magkos2020diet}, especially weight control, prove to be the most effective therapy, particularly for those in the prediabetic condition. While these therapeutic interventions have demonstrated success in controlling blood glucose levels within specific patient subgroups, their effectiveness varie among individuals. Recent technological advances in medical technology offer new opportunities for continuous patient monitoring and personalized interventions \cite{matabuena2024multilevel}. In the case of glucose metabolism, Continuous Glucose Monitoring (CGM) devices have emerged as revolutionary tools, providing real-time data on glucose levels over time. In diabetic populations, CGM has become the gold standard to optimize individual glycemic control. In healthy populations, the use of CGM is gaining attention in the field of personalized  nutrition to improve individual metabolic capacity. Overall, technology enables a more personalized approach to diabetes care and improves metabolic status in healthy populations. However, from a statistical point of view, the analysis of CGM data is not straightforward. First, the quasi-continuous observations at different time scales are functional and longitudinal. In addition, patients in real-world conditions are monitored in free-living environments, and standard time series techniques are not directly applicable due to chronobiological and registration information differences \cite{matabuena2021glucodensities}. Common practice in the field of digital health involves creating raw time series scalar summaries such as the mean and standard deviation of glucose time series, but several studies have shown that this approach discards a significant amount of individual information about glycemic conditions.

To underscore the versatility and potential of our variable selection framework in medical research, particularly in the context of DM, we address various research questions posed by clinical researchers in the field of diabetes. Table~\ref{tab:research_questions} outlines the different problems examined. The analysis primarily focuses on novel biomarkers derived from continuous patient monitoring, such as functional data and probability distributions as well as another graph structures. We also focus  on more standard settings involving multivariate data, in order to demonstrates the effectiveness and scalability properties of our proposed framework in large cohorts. In the main manuscript, we only focus in the first three examples.

\begin{table}[ht!]
    \centering
    \footnotesize 
    \setlength{\tabcolsep}{8pt} 
    \begin{tabular}{|c|p{0.6\linewidth}|>{\raggedright\arraybackslash}p{0.3\linewidth}|}
        \hline
        \rowcolor{lightgray}
        \textbf{No.} & \textbf{Research Questions} & \textbf{Prediction Target} \\
        \hline
        1 & Utilizing a large list of clinical biomarkers, can we identify a reduced subset of $k$-biomarkers to predict the clinical outcomes used to diagnosing and monitoring the progression of diabetes mellitus, such as glycosylated hemoglobin and fasting plasma glucose? & Biochemical markers (e.g., glycosylated hemoglobin, fasting plasma glucose) \\
        
        \hline
    2 & With the increasing profileration and importance of digital health and Continuous Glucose Monitoring (CGM) in clinical practice, can we predict glucose mean and variance derived from the CGM over diferent periods in clinical trials by utilizing baseline biomarkers alongside traditional surrogate measures? & CGM-related measures (e.g., glucose mean and variance over time) \\
        \hline
      3 & Recent digital health research introduce functional digital biomarkers like  marginal density of  CGM time series representations. Can we predict this functional representation at the end of a clinical trial using baseline biomarkers? & Functional representation of glucose time series. \\
        \hline
    4 & In order to introduce more sophisticated biomarkers that introduce the temporal dismension of CGM in the model, we aim to predict the temporal quantile CGM response utilizing as a predictor a histogram that contain baseline CGM data. & Multivariate-histogram-based predictor for functional temporal quantile distributional profiles \\
        \hline

        5 & Considering the known statistical association of diet and physical activity levels in the glycaemic levels, can we predict variables indicating diabetes status, physical activity, and dietary profiles simultaneously from a diverse set of biomarkers? & Diabetes risk, physical activity, and dietary profiles (different sources of information) \\
\hline

        6 & In the domain of diabetes research, the CGM data of each individual can be summarize by a graph measuring the similarity between CGM profiles across different days of the week. Can we predict the graph biomarkers estimated at the end of the intervention using baseline data? & Graph-based biomarker capturing intra and inter-day individual glucose variability \\
        \hline
    \end{tabular}
    \caption{Research Questions and Prediction Targets}
    \label{tab:research_questions}
\end{table}

\subsection{Multivariate vectorial Euclidean data}
Surely, one of the most prevalent non-standard scenarios in medical research is the simultaneous analysis of multiple scalar biomarkers as a multivariate vector. In medical practice and diagnosis, the accurate presence of diseases often necessitates multiple diagnostic criteria. For example, in the case of DM, physicians often consider glycosylated hemoglobin (A1C) and fasting plasma glucose (FPG) biomarkers simultaneously \cite{Selvin2011}. A1C is a stable biomarker that reflects the average glucose level over the previous three months, while FPG captures glucose fasting levels in the morning on a particular day and shows greater intra-day variability but measures aspects that A1C cannot \cite{selvin2007short}. In certain cases, such as gestational diabetes, other alternative and more complex biomarkers like an oral glucose tolerance test are considered in the clinical diagnosis.

Given the critical importance and technical challenges in developing reliable models to determine patients' clinical status and accurately characterize diseases phenotypically, it is essential to select the $k$-variables that maximize the statistical associations with multivariate responses. To address this scientific problem, we propose solving the following optimization problem:

\begin{align}\label{eqn:multivariate} 
\min_{\substack{\beta \in \mathbb{R}^{p \times m} \\ s \in \{0,1\}^{p}}} \quad & \sum_{t=1}^{m} \sum_{i=1}^{n} \frac{1}{2}(Y_{it}-\langle X_{i},\beta_{\cdot t} \rangle )^{2} + \frac{1}{2\gamma} \sum_{t=1}^{m} \|\beta_{\cdot t}\|_{2}^{2} \\
\text{s.t.} \quad & \sum_{j=1}^{p} s_{j} \leq k, \\
& \beta_{jt} = 0 \text{ if } s_j = 0, \, \forall j \in [p] , t \in [m]. \nonumber
\end{align}

From a statistical modeling point of view, the study of multivariate Euclidean data is important because it allows handling more complex scenarios. For example, given a separable Hilbert space $\mathcal{H}$, suppose that for any $x \in \mathcal{H}$, $x \approx \sum_{j=1}^{m} c_{xj} \phi_{j}$, where $c_{x} = (c_{x1}, \dots, c_{xm}) \in \mathbb{R}^{m}$ and $\{\phi_{j}\}_{j=1}^{m}$ is an orthogonal basis of functions from $\mathcal{H}$. Then, many regression problems that involve response and predictors in the functional space $\mathcal{H}$ can be expressed directly as multivariate regression problems, in which each datum $x \in \mathcal{H}$ is summarized by a finite-dimensional vector $c_{x}$. This is typical, for example, in the field of functional data when $\mathcal{H} = L^2([0,1])$ as we approximate each datum $x \in \mathcal{H}$ as $x \approx \widehat{c}_{x}$, where $\widehat{c}_{x}$ represents the first $m$ terms of the Karhunen-Loève expansion \cite{daw2022overview}
 (PCA for functional data analysis method in this case).

The literature on multivariate Euclidean data offers various alternatives for variable selection in such environments (see for example \cite{sofer2014variable}), serving as a natural progression from the lasso and other popular linear regularized models. However, there are a limited number of algorithms that effectively handle categorical variables while maintaining computational efficiency. Furthermore, many of these methods aim to improve statistical efficiency by incorporating the correlation structure of random errors among response variables, but they are not computationally scalable and cannot easily handle other loss functions as our case.
In contrast, our method scales effectively for large sample sizes and maintains statistical consistency under certain technical regularity conditions. Consequently, we do not anticipate significant gains in statistical efficiency by introducing the correlation structure of random errors in large-scale datasets.  Given these technical considerations, our paper focuses exclusively on evaluating our algorithm using a real-world example.

\subsubsection{Data Description}

For the prediction of diabetes biomarkers A1C and FPG, we utilized data from the NHANES 2002-2018 cohort \cite{curtin2013national}, which provides a comprehensive list of clinical variables. The datasets include both interview and physical examination data, capturing demographic, and biochemical variables. Overall, we analyzed a random sample of \( n = 58,000 \) individuals and evaluated \( p = 27 \) candidate predictors for the pool of best biomarkers. For a detailed description of the variables used in this analysis, we refer readers to the Supplemental Material.

\subsubsection{Results}

After applying the models described in Equation \ref{eqn:multivariate}, we selected a total of \( k = 9 \) biomarkers, which is one-third of the total number of biomarkers available in the initial pool of variables. The algorithm obtained the optimal solution in just two seconds.

The \( R^2 \) values for A1C and FPG are \( 0.26 \) and \( 0.17 \), respectively from the variable selection model, while in the original model with 27 variables, the \( R^2 \) values are \( 0.28 \) and \( 0.19 \) when evaluated on the training sample. This example demonstrates how the variable selection process can be useful in discarding irrelevant variables and creating more interpretable models.

\subsection{Longitudinal and functional Euclidean data}

Over the past decade, driven by technological advancements, there has been a growing interest in analyzing longitudinal outcomes, facilitated by the proliferation of digital health technologies and electronic health records \cite{jiang2018dynamic, li2017digital, Chen2018, zhou2019longitudinal}. The current capability to collect genetic profiles as omic data alongside physiological information over time presents new opportunities for creating detailed phenotypic characterizations to support medical decisions from a precision medicine perspective.

In the statistical literature, one of the first papers to focus on analyzing longitudinal clinical outcomes with high-dimensional predictors is by \cite{barber2017function}. Recent years have seen new contributions in this direction \cite{parodi2018simultaneous, reimherr2019optimal, fan2017high, mirshani2021adaptive}, primarily concentrating on the conditional mean perspective, linear regression models, and both binary and continuous predictors.

 We briefly introduce the mathematical models from \cite{barber2017function}. Consider a random variable response \( Y \in \mathcal{Y} = L^{2}([0,1]) \) and the following regression model:

\begin{equation}
    Y(t) = \langle \beta(t), X \rangle + \epsilon(t), \quad t \in [0,1].
\end{equation}

\noindent Here, $X \in \mathbb{R}^{p}$, $\beta: [0,1] \to \mathbb{R}^{p}$ represents the dynamic slope function, and $\epsilon(t)$ denotes the random functional error. For identifiability, we assume:
\begin{enumerate}
    \item $\mathbb{E}\bigl[ \epsilon(t) \mid X \bigr] = 0$,
    \item $\operatorname{Cov}\bigl( \epsilon(t), \epsilon(s) \mid X \bigr) = \Sigma(t, s) \in \mathbb{R}$ for all $t, s \in [0,1]$,
\end{enumerate}
\noindent where $\Sigma \in \mathcal{Y} \otimes \mathcal{Y}$ is a covariance operator on the space $\mathcal{Y} = L^{2}([0,1])$.

For practical purposes, we consider observations of the random variable $Y$ at specific time points within an equispaced grid $\Gamma_{m} = \left\{ t_{1} = \dfrac{1}{m},\ t_{2} = \dfrac{2}{m},\ \dots,\ t_{m} = 1 \right\}$. We observe random samples $\{ (X_i, Y_i) \}_{i=1}^{n}$, where for each $i$:
\begin{itemize}
    \item $X_i = \bigl( X_{i,1},\ X_{i,2},\ \dots,\ X_{i,p} \bigr) \in \mathbb{R}^{p}$ is the vector of covariates,
    \item $Y_i = \bigl( Y_i(t_1),\ Y_i(t_2),\ \dots,\ Y_i(t_m) \bigr) \in \mathbb{R}^{m}$ are the corresponding response variables.
\end{itemize}

For variable selection, the original reference \cite{barber2017function} employs the strategy described below:

\begin{align}
 \min_{\substack{\beta= (\beta(1),\cdots, \beta(p))^{\top} \in \mathbb{R}^{p \times m}}} \quad &\sum_{t=1}^{m} \sum_{i=1}^{n} \frac{1}{2}\left(Y_{i}(t) - \langle X_{i}, \beta(t) \rangle \right)^{2} + \frac{1}{2\gamma} \sum_{t=1}^{m} \|\beta(t)\|_{2}^{2}
\end{align}

Following our general formulation, we propose solving:

\begin{align}
 \min_{\substack{\beta= (\beta(1),\cdots, \beta(p))^{\top} \in \mathbb{R}^{p \times m} \\ s \in \{0,1\}^{p}}} \quad &\sum_{t=1}^{m} \sum_{i=1}^{n} \frac{1}{2}\left(Y_{i}(t) - \langle X_{i}, \beta(t) \rangle \right)^{2} +  \frac{1}{2\gamma} \sum_{t=1}^{m} \|\beta(t)\|_{2}^{2} \\
    \text{s.t.} \quad &\sum_{j=1}^{p} s_{j} \leq k, \\
    &\beta_{j}(t_r) = 0, \quad \forall r \in \{1,\dots,p\} \text{ if } s_j = 0. \nonumber
\end{align}

\subsubsection{Data description}

To demonstrate the versatility and efficacy of our framework, we analyze data from the Juvenile Diabetes multicenter study \cite{doi:10.1056/NEJMoa0805017}. This study evaluates the effectiveness and safety of continuous glucose monitoring (CGM) devices in managing glucose levels among diabetes patients. In this randomized trial, 322 individuals, both adults and children receiving intensive therapy for type 1 diabetes, were divided into two groups: one utilizing continuous glucose monitoring and a control group using traditional home monitoring with blood glucose meters. Participants were further stratified into three age groups, with glycated hemoglobin levels ranging from 7.0\% to 10.0\%. The primary clinical outcome measured was the change in glycated hemoglobin levels over 26 weeks \cite{doi:10.1056/NEJMoa0805017}.

Our analysis centers on the mean glucose levels and the standard deviations of glucose measurements across seven distinct time periods, focusing on clinical outcomes. We included only patients with complete data sets in our analysis. Our goal is to identify the baseline variables that significantly influence mean glucose levels and glucose variability during the initial assessment period. To this end, we examined a range of baseline variables, including age, sex, and another patient clinical characteristics. The practical application of our research lies in its potential to identify cost-effective surrogate biomarkers for monitoring the progression of diabetes mellitus disease.

\subsubsection{Results}
We split the 26-week period into 7 intervals. For each interval, we estimated the mean Continuous Glucose Monitoring (CGM) values and the standard deviation. Figure \ref{fig:long} shows the results of the resulting trajectories for all participants. Using these two longitudinal profiles, we applied variable selection methods to identify the three most relevant variables.

The predictors introduced in the model include baseline variables such as gender, age, height, and weight, as well as CGM-derived metrics like the proportion of hypoglycemia, glucose mean, and hyperglycemia. The variables selected were weight, height, and the hyperglycemia range. However, if we run the algorithm using only the longitudinal average glucose profile, the selected variables change, discarding height in favor of the hypoglycemia range. This results highlight the global nature of the algorithm in the case of use as a outcome a bidimensional--longitudinal profile.

\begin{figure}[ht]
	\centering
	\includegraphics[width=0.9\linewidth]{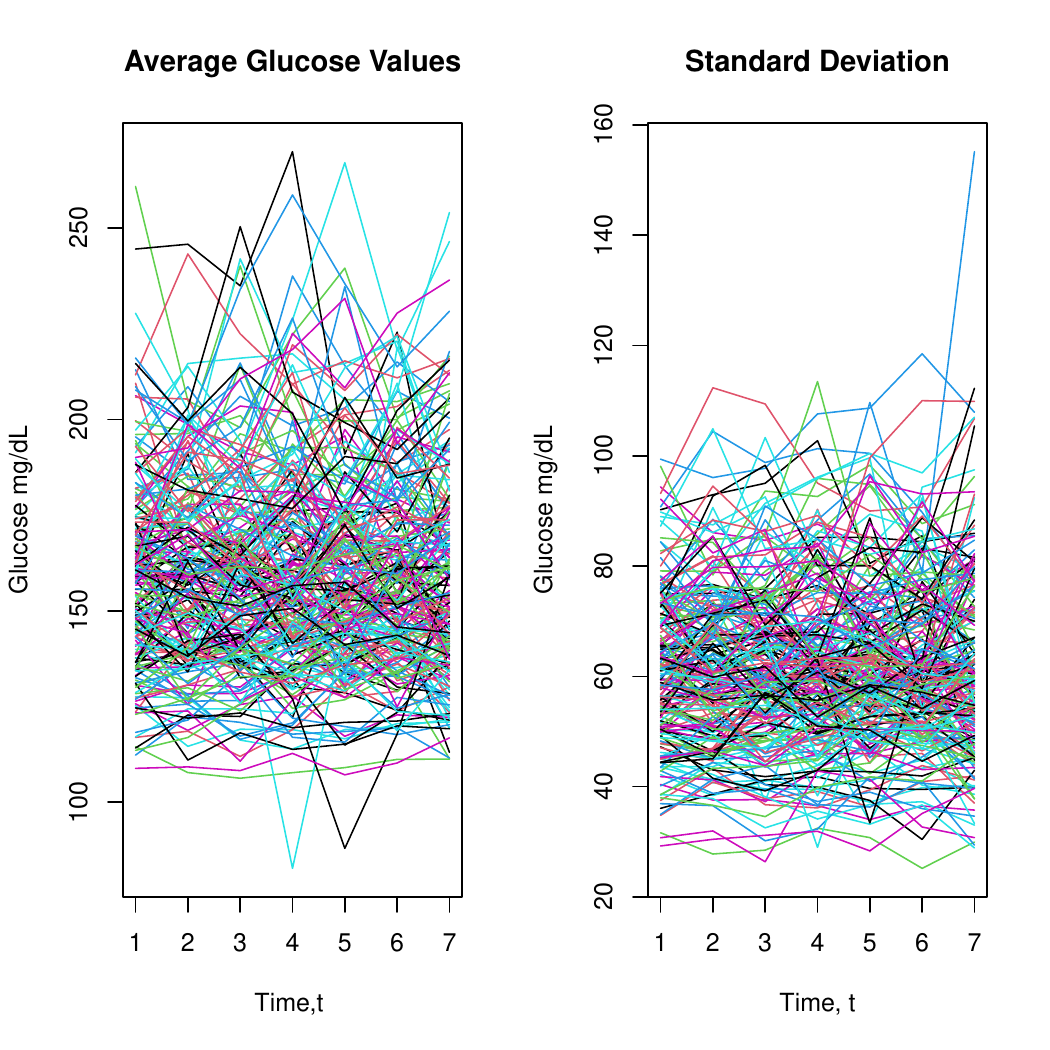}
	\caption{Average glucose trajectories (left) and standard deviation trejectories (right).}
	\label{fig:long}
\end{figure}
\subsection{Distributional representation as a clinical outcome in diabetes research}

\begin{figure}[ht]
	\centering
	\includegraphics[width=0.9\linewidth]{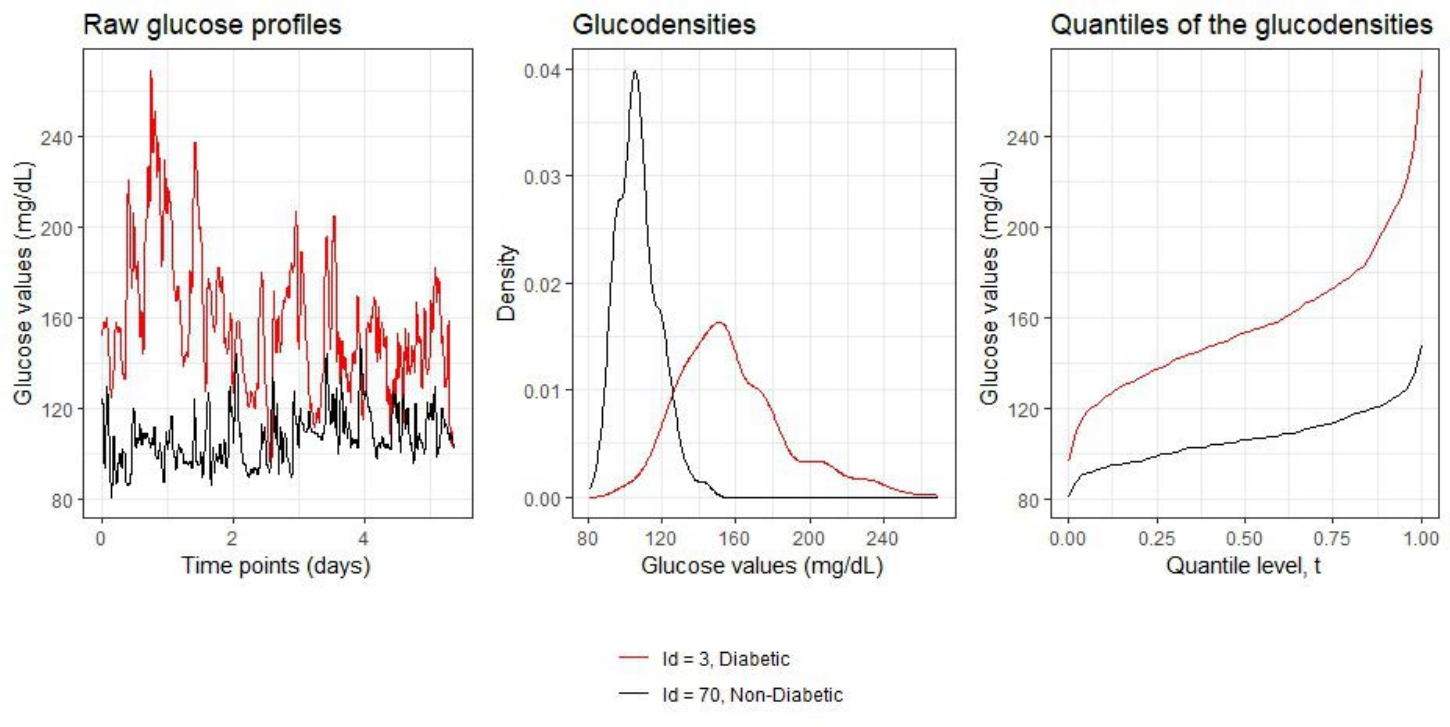}
	\caption{Left: Raw CGM time series of two individuals. Center: The corresponding density functions. Right: The corresponding quantile representation.}
	\label{fig:gluco_overview}
\end{figure}

Distributional data analysis is a cutting-edge methodology for analyzing digital health time series data obtained from various healthcare devices, such as accelerometers, continuous glucose monitors, or heart rate monitors \cite{ghosal2021distributional, matabuena2021distributional}. This approach focuses on analyzing the distributional patterns of biosensor time series data in the form of density, probability distribution, or quantile functions.

Currently, distributional representations have been successfully applied in several clinical domains, including functional magnetic resonance imaging (fMRI) for brain network connectivity analysis \cite{petersen2016functional}, diabetes \cite{matabuena2021glucodensities}, and physical activity analysis. These representations serve as efficient patient outcomes that contain more information about patients' biological processes than other summary metrics commonly used in digital health, such as compositional vector-valued metrics that collapse the distributional pattern of time series into intervals. Another significant advantage is the interpretability of distributional representations in clinical applications, unlike other latent representations derived from time series using neural network models. However, a technical limitation of traditional distributional representations is their omission of the temporal sequence in time series events, as they focus only on marginal distributional patterns.

From a mathematical perspective, distributional data can be considered a subtype of functional data analysis. However, it comes with geometric constraints because density or probability functions are not defined in linear spaces with a vector-valued structure.

Now, let's define the technical details of distributional representations based on time series data. Given a time series for biosensor devices \(\{Y_j\}_{j=1}^m\), the distributional representation can be modeled as a probability density function \(f(\cdot)\) that can be estimated by kernel density estimation:

\begin{equation}
 \widehat{f}(y)=\frac{1}{m}\sum_{j=1}^{m} \frac{1}{h} K(\frac{Y_j-y}{h}),
\end{equation}

\noindent where \(h>0\) is the smoothing parameter and \(K\left(\cdot\right)\) denotes a non-negative real-valued integrable function (Figure \ref{fig:gluco_overview}).

Let \(\mathcal{D}\) be the space of probability density functions \(f\) such that \(\int _{\mathbb{R}}u^2f(u)du <\infty\). To measure the difference between two density functions, \(f\) and \(g\), a metric on \(\mathcal{D}\) is required. We use the \(2\)-Wasserstein distance:

\begin{equation}
d^2_{\mathcal{W}_2}(f,g)= \int_{0}^{1} \left|Q_{f}\left(t\right)-Q_{g}\left(t\right)\right|^{2}dt,  \quad f, g\in \mathcal{D},
\label{eq:wasserstein}
\end{equation}

\noindent In the context of computing the \(2\)-Wasserstein distance, denoted by \(d^2_{\mathcal{W}_2}\) for univariate probability, it only involves the quantile functions \(Q_f\) and \(Q_g\). From a practical perspective, to approximate the \(2\)-Wasserstein distance for the \(i\)-th individual, we observe the time-series observations \(\{Y_{ij}\}_{j=1}^{n_i}\) and, based on the empirical distribution 
\[
\widehat{F}_{f_i}(t) = \frac{1}{n_i} \sum_{j=1}^{n_i} \mathbb{I}\left(Y_{ij} \leq t \right),
\]
we obtain the empirical quantile function denoted as \(\widehat{Q}_{f_i}(t)\).

A commonly adopted approach in distributional data analysis literature is to leverage the quantile representation, which involves embedding probability distributions or density functions in a separable Hilbert space. Suppose that \(Y_{i}(t)= \widehat{Q}_{i}(t)\), \(t\in [0,1]\), and that we observe the probabilities of the quantiles on a grid \(\Gamma = \{t_j\}_{j=1}^{m}\), where \(0=t_1<t_2<\dots<t_m=1\). Then, we can formulate the variable selection problem as:

\begin{align}
 \min_{\substack{\beta= (\beta(t_1),\cdots, \beta(t_m))^{\top} \in \mathbb{R}^{p\times m} \\ s \in \{0,1\}^{p}}} \quad &\sum_{r=1}^{m} \sum_{i=1}^{n} \frac{1}{2}\left(Y_{i}(t_r)-\langle X_{i},\beta(t_r) \rangle \right)^{2} +  \frac{1}{2\gamma} \sum_{r=1}^{m} \|\beta(t_r)\|_{2}^{2}, \\
    \text{s.t.} \quad &\sum_{r=1}^{p} s_{j} \leq k, \\
    &\beta_{j}(t_r) = 0, \forall r \in \{1,\dots,m\} \text{ if } s_j = 0. \nonumber
\end{align}

\subsubsection{Data description}

The increasing aging of modern societies and the growing number of older adults with type 1 diabetes (T1D) represent significant challenges from a public health perspective. Older adults with T1D are more vulnerable to severe hypoglycemia, which can lead to dangerous complications such as altered mental status, seizures, cardiac arrhythmias, and even higher mortality rates.

In this section, we analyze data from \cite{10.2337/dc15-1426}—a continuous glucose monitoring (CGM) cohort study—to examine hypoglycemia events in aging populations. Previous research has highlighted the prevalence of severe hypoglycemia in this population. Notably, it has been found that severe hypoglycemia occurs among adults with T1D at similar rates, regardless of their glycated hemoglobin (HbA1c) levels. This challenges the conventional treatment strategy, which often prioritizes reducing hypoglycemia risk over achieving lower HbA1c levels.

Despite the recognized risk of severe hypoglycemia in older adults with longstanding T1D, there is a significant gap in research exploring the factors contributing to its development, especially using distributional data analysis tools. This approach captures individual profiles across hypoglycemia, hyperglycemia, and normal glucose ranges, rather than focusing solely on specific distributional metrics from CGM time series. The main goal of this section is to address this gap in the literature.

The dataset analyzed is a  case-control study, which includes 200 participants—100 cases of severe hypoglycemia and 100 controls—sourced from 18 diabetes centers.

\subsubsection{Results}
For the 200 participants, we represent the clinical outcome \( Y_{i}(t) = \hat{\mathcal{Q}}_{i}(t) \), where \( t \in [0,1] \), on a grid of  $500$ equispaced points. The predictors include the binary variables: whether the patient lives alone (\textit{LiveAlone}) and gender, as well as C-peptide status, which indicates if the body is producing insulin. Additionally, we consider the continuous variables: glucose, hemoglobin glucose (HbA1c), serum creatinine, and weight, resulting in a total of 7 variables.

These predictors are used in the prediction of the quantile glucose outcome (see Figure \ref{fig:figura1}). We run the algorithms and we select three variables, which include the continuous biomarkers: glucose, HbA1c, and weight. Figure \ref{fig:figura2} shows the marginal \( p \)-values for testing the null hypothesis \( H_{0}: \beta_{r}(t) = 0, \, r \in \{ \text{Glucose, HbA1c, Weight} \} \) for each \( t \in [0,1] \). The results indicate that the importance of these variables is not consistent across the entire domain of the quantile function. For example, weight is only relevant in the hypoglycemia range (\text{Percentile} < 20), while HbA1c is clinically relevant across the whole domain of quantile function. These findings highlight the global nature of the variable selection method.

\begin{figure}[h]
    \centering
    \begin{minipage}{0.45\textwidth}
        \centering
        \includegraphics[width=\linewidth]{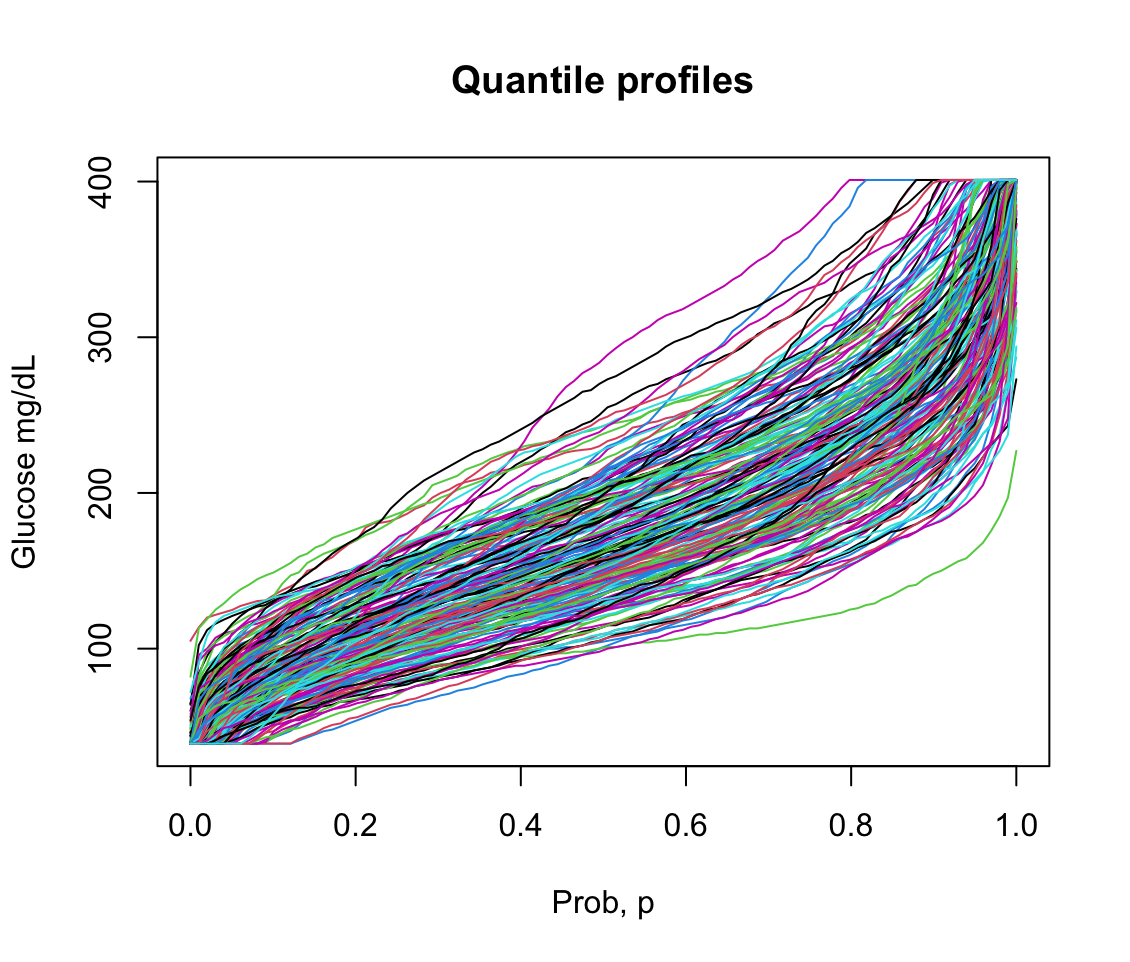}
        \caption{Raw Quantile Outcomes}
        \label{fig:figura1}
    \end{minipage}%
    \hfill
    \begin{minipage}{0.45\textwidth}
        \centering
        \includegraphics[width=\linewidth]{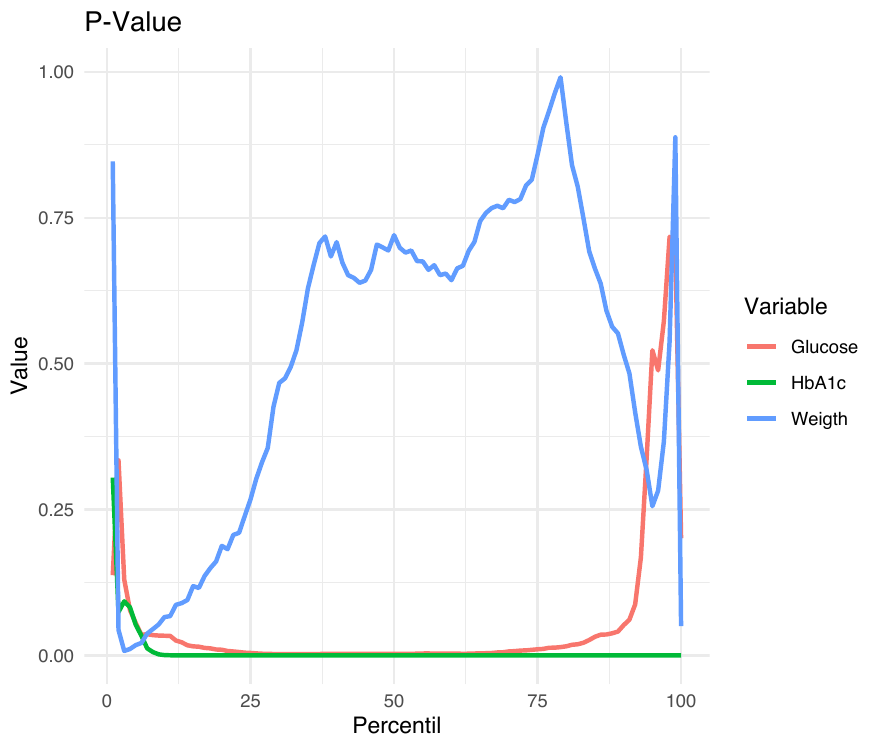}
        \caption{P-values across the temporal domain of the statistical significance of each variable selected}
        \label{fig:figura2}
    \end{minipage}
    \caption{Figures of the statistical analysis for distributional representations}
\end{figure}

\section{Simulation Study}\label{sec:sim}
To examine and analyze the computational efficiency and scalability of our methods in a controlled environment without sacrificing statistical accuracy, we focus on two simulation scenarios for different random responses $Y$: (i) multivariate Euclidean data, and (ii) probability distributions equipped with the $2$-Wasserstein metric. For  (i), we assess the empirical performance of our algorithm with large sample sizes, up to $n = 1000{,}
000$. Due to the computational limitations of competing methods, we do not include a comparison with other approaches in this scenario. In  (ii), we restrict our comparison to the only generic alternative for variable selection with metric space responses, \texttt{FRISO} \cite{tu2020era}. For each scenario, we conduct 200 simulations. We report   execution time and statistical performance metrics.

\subsection{Multivariate Euclidean Responses}
As a first evaluation task, we focus on the multivariate Euclidean responses under the lens of linear regression model. 

Consider the regression model:

\begin{equation}
Y_{ij} = \sum_{r=1}^{p} X_{ir}\beta_{rj} + \epsilon_{ij}, \quad \text{for } i\in [n] , \text{ and } j \in [m],
\end{equation}

\noindent where $\epsilon_{i}=(\epsilon_{i1}, \dots, \epsilon_{im})^{\top}\sim \mathcal{N}(0,\Sigma_{\mathcal{Y}})$ is a multivariate Gaussian random distribution. We assume $X_{i} \sim \mathcal{N}(0,\Sigma_{\mathcal{X}})$, and $\beta_{rj} = \texttt{effect}\in \{0.1,0.5,1\}$ if $r \in \mathcal{S}_{true} \subset \{1, \dots, p\}$, and zero otherwise, for $j=1,2,\dots,m$ and $r=1, \dots, p$. In practice, $n \in \{200, 2000, 20000, 100000\}$, $p \in \{5, 20,  50\}$, and $|\mathcal{S}_{true}| = 2$, and $m\in \{3,20, 50\}$, with $B = 200$ simulation for each scenario. We assume covariance matrices $\Sigma_{\mathcal{X}}$ and $\Sigma_{\mathcal{Y}}$ are Toeplitz matrix of the form:

\[
\Sigma_{\mathcal{X}} = \begin{pmatrix}
1 & \rho_{\mathcal{X}} & \rho_{\mathcal{X}} & \cdots & \rho_{\mathcal{X}} \\
\rho_{\mathcal{X}} & 1 & \rho_{\mathcal{X}} & \cdots & \rho_{\mathcal{X}} \\
\rho_{\mathcal{X}} & \rho_{\mathcal{X}} & 1 & \cdots & \rho_{\mathcal{X}} \\
\vdots & \vdots & \vdots & \ddots & \vdots \\
\rho_{\mathcal{X}} & \rho_{\mathcal{X}} & \rho_{\mathcal{X}} & \cdots & 1
\end{pmatrix},
\]

\noindent, where $\rho_{\mathcal{X}} \in \{0,0.6\}$ and $\rho_{\mathcal{Y}} \in \{0,0.6\}$ are two positive constants. Consequently, there are four scenarios in our analysis. In the first scenario, there is no correlation between predictors and the response, with $\rho_{\mathcal{X}} = \rho_{\mathcal{Y}} = 0$. In the second scenario, we assume a strong correlation between predictors and response, with $\rho_{\mathcal{X}} = \rho_{\mathcal{Y}} = 0.6$. The third scenario assumes $\rho_{\mathcal{X}} = 0.6$, and , $\rho_{\mathcal{Y}} = 0$. In the fourth scenario, the predictors are independent, $\rho_{\mathcal{X}} = 0$, while a strong correlation exists within the response components, $\rho_{\mathcal{Y}} = 0.6$.

Since the main purpose of this analysis is to test the computational performance of our methods in a fixed scenario, we run the methods to identify a total of $|\mathcal{S}_{\text{true}}| = 2$ most relevant variables in each simulation, corresponding to the ground truth. In each simulation, $b = 1, \dots, B = 200$, we record the execution time $t_b$. We also check if the model correctly identifies the true model, denoted as $M_{b, \text{true}}$, defined by the condition:

\[
M_{b, \text{true}} = \mathbb{I}\left\{ \left| \{j \in \{1, \dots, p\} : \beta_{rj} \neq 0 \text{ and } \widehat{\beta}_{rj} \neq 0 \text{ for any } r \in \{1, \dots, m\} \} \right| = |\mathcal{S}_{\text{true}}| = 2 \right\}.
\]

Additionally, we compute two types of errors after fitting a linear regression model using the selected variables:

\begin{itemize}
    \item The maximum error in the estimated $\beta$ coefficients:
    \[
    \epsilon_{\infty, b} = \max_{j=1,\dots,m, \, r=1,\dots,p} \left|\beta_{rj} - \widehat{\beta}_{rj}\right|.
    \]

    \item The empirical $L^{1}$ error:
    \[
    \epsilon_{1, b} = \frac{1}{2 \cdot m} \sum_{r=1}^{p} \sum_{j=1}^{m} \left|\beta_{rj} - \widehat{\beta}_{rj}\right|.
    \]
\end{itemize}

We summarize the results across $B=200$ simulations for the different metrics in terms of the average results. In the  variables 
$\texttt{emax}$ and $\texttt{eaverage}$ corresponding to the results to variables $\epsilon_{\infty,b}$ and $\epsilon_{1,b}$,  we report $mean\pm sd$.

\begin{table}[H]
\centering
\scriptsize 
\begin{tabular}{|r|r|r|r|r|l|l|r|r|}
  \hline
  \texttt{p} & \texttt{$\rho_x$} & \texttt{$\rho_y$} & \texttt{effect} & \texttt{eaverage} & \texttt{emax} & \texttt{time} & \texttt{correct} \\ 
  \hline
  5  & 0.00 & 0.00 & 0.10 & 0.0637 ± 0.0320 / 0.003 ± 0.0002 & 0.168 ± 0.0496 / 0.0083 ± 0.0013 & 0.06 / 1.38 & 1.00 / 1.00 \\ 
  5  & 0.00 & 0.00 & 0.50 & 0.1747 ± 0.1631 / 0.0029 ± 0.0003 & 0.352 ± 0.2581 / 0.0083 ± 0.0011 & 0.06 / 1.41 & 1.00 / 1.00 \\ 
  5  & 0.00 & 0.00 & 1.00 & 0.0506 ± 0.0054 / 0.0041 ± 0.0002 & 0.141 ± 0.0227 / 0.0099 ± 0.0013 & 0.05 / 0.86 & 1.00 / 1.00 \\ 
  5  & 0.00 & 0.60 & 0.10 & 0.0660 ± 0.0304 / 0.0027 ± 0.0007 & 0.142 ± 0.0443 / 0.0071 ± 0.0014 & 0.06 / 1.39 & 0.90 / 1.00 \\ 
  5  & 0.00 & 0.60 & 0.50 & 0.2238 ± 0.1022 / 0.0030 ± 0.0010 & 0.464 ± 0.2090 / 0.0075 ± 0.0017 & 0.06 / 1.39 & 1.00 / 1.00 \\ 
  5  & 0.00 & 0.60 & 1.00 & 0.0620 ± 0.0221 / 0.0034 ± 0.0009 & 0.129 ± 0.0283 / 0.0079 ± 0.0019 & 0.05 / 0.97 & 1.00 / 1.00 \\ 
  5  & 0.60 & 0.00 & 0.10 & 0.0469 ± 0.0113 / 0.0032 ± 0.0004 & 0.146 ± 0.0239 / 0.0092 ± 0.0015 & 0.06 / 1.65 & 1.00 / 1.00 \\ 
  5  & 0.60 & 0.00 & 0.50 & 0.0452 ± 0.0046 / 0.0032 ± 0.0004 & 0.137 ± 0.0240 / 0.0108 ± 0.0014 & 0.06 / 1.37 & 1.00 / 1.00 \\ 
  5  & 0.60 & 0.00 & 1.00 & 0.0460 ± 0.0048 / 0.0034 ± 0.0005 & 0.139 ± 0.0240 / 0.0103 ± 0.0026 & 0.03 / 0.95 & 1.00 / 1.00 \\ 
  5  & 0.60 & 0.60 & 0.10 & 0.0513 ± 0.0175 / 0.0030 ± 0.0010 & 0.124 ± 0.0329 / 0.0079 ± 0.0022 & 0.06 / 1.53 & 0.60 / 1.00 \\ 
  5  & 0.60 & 0.60 & 0.50 & 0.1421 ± 0.0843 / 0.0033 ± 0.0008 & 0.248 ± 0.1155 / 0.0084 ± 0.0019 & 0.05 / 0.95 & 1.00 / 1.00 \\ 
  5  & 0.60 & 0.60 & 1.00 & 0.0746 ± 0.1040 / 0.0031 ± 0.0009 & 0.152 ± 0.1428 / 0.0080 ± 0.0022 & 0.03 / 0.81 & 1.00 / 1.00 \\ 
  20 & 0.00 & 0.00 & 0.10 & 0.0449 ± 0.0151 / 0.0026 ± 0.0003 & 0.123 ± 0.0389 / 0.0074 ± 0.0013 & 0.12 / 12.39 & 1.00 / 1.00 \\ 
  20 & 0.00 & 0.00 & 0.50 & 0.0367 ± 0.0031 / 0.0029 ± 0.0004 & 0.110 ± 0.0132 / 0.0087 ± 0.0013 & 0.12 / 9.50 & 1.00 / 1.00 \\ 
  20 & 0.00 & 0.00 & 1.00 & 0.0493 ± 0.0049 / 0.0038 ± 0.0004 & 0.132 ± 0.0127 / 0.0103 ± 0.0016 & 0.07 / 1.66 & 1.00 / 1.00 \\ 
  20 & 0.00 & 0.60 & 0.10 & 0.0491 ± 0.0293 / 0.0023 ± 0.0006 & 0.120 ± 0.0623 / 0.0065 ± 0.0012 & 0.12 / 11.91 & 0.80 / 1.00 \\ 
  20 & 0.00 & 0.60 & 0.50 & 0.3299 ± 0.1907 / 0.0027 ± 0.0006 & 0.496 ± 0.2167 / 0.0070 ± 0.0014 & 0.09 / 9.06 & 1.00 / 1.00 \\ 
  20 & 0.00 & 0.60 & 1.00 & 0.1835 ± 0.2181 / 0.0030 ± 0.0013 & 0.385 ± 0.4349 / 0.0073 ± 0.0020 & 0.06 / 2.78 & 1.00 / 1.00 \\ 
  20 & 0.60 & 0.00 & 0.10 & 0.0410 ± 0.0076 / 0.0033 ± 0.0002 & 0.134 ± 0.0281 / 0.0100 ± 0.0018 & 0.12 / 14.20 & 1.00 / 1.00 \\ 
  20 & 0.60 & 0.00 & 0.50 & 0.0410 ± 0.0064 / 0.0033 ± 0.0003 & 0.132 ± 0.0167 / 0.0105 ± 0.0017 & 0.09 / 5.65 & 1.00 / 1.00 \\ 
  20 & 0.60 & 0.00 & 1.00 & 0.0442 ± 0.0051 / 0.0033 ± 0.0005 & 0.139 ± 0.0205 / 0.0100 ± 0.0023 & 0.06 / 1.64 & 1.00 / 1.00 \\ 
  20 & 0.60 & 0.60 & 0.10 & 0.0404 ± 0.0223 / 0.0026 ± 0.0005 & 0.104 ± 0.0349 / 0.0082 ± 0.0021 & 0.09 / 11.42 & 0.40 / 1.00 \\ 
  20 & 0.60 & 0.60 & 0.50 & 0.0838 ± 0.0687 / 0.0037 ± 0.0014 & 0.164 ± 0.0949 / 0.0088 ± 0.0028 & 0.07 / 4.06 & 1.00 / 1.00 \\ 
  20 & 0.60 & 0.60 & 1.00 & 0.1476 ± 0.1797 / 0.0044 ± 0.0014 & 0.249 ± 0.2360 / 0.0099 ± 0.0023 & 0.06 / 1.29 & 1.00 / 1.00 \\ 
  50 & 0.00 & 0.00 & 0.10 & 0.1027 ± 0.0067 / 0.0025 ± 0.0003 & 0.180 ± 0.0118 / 0.0080 ± 0.0019 & 0.17 / 18.11 & 1.00 / 1.00 \\ 
  50 & 0.00 & 0.00 & 0.50 & 0.0382 ± 0.0052 / 0.0029 ± 0.0003 & 0.107 ± 0.0207 / 0.0087 ± 0.0020 & 0.15 / 19.43 & 1.00 / 1.00 \\ 
  50 & 0.00 & 0.00 & 1.00 & 0.0506 ± 0.0059 / 0.0034 ± 0.0005 & 0.142 ± 0.0295 / 0.0093 ± 0.0010 & 0.11 / 4.79 & 1.00 / 1.00 \\ 
  50 & 0.00 & 0.60 & 0.10 & 0.0942 ± 0.0045 / 0.0024 ± 0.0008 & 0.143 ± 0.0078 / 0.0064 ± 0.0013 & 0.17 / 19.94 & 0.30 / 1.00 \\ 
  50 & 0.00 & 0.60 & 0.50 & 0.0381 ± 0.0160 / 0.0027 ± 0.0010 & 0.100 ± 0.0260 / 0.0077 ± 0.0028 & 0.12 / 19.58 & 1.00 / 1.00 \\ 
  50 & 0.00 & 0.60 & 1.00 & 0.2541 ± 0.2704 / 0.0036 ± 0.0012 & 0.515 ± 0.5287 / 0.0092 ± 0.0028 & 0.10 / 8.20 & 1.00 / 1.00 \\ 
  50 & 0.60 & 0.00 & 0.10 & 0.0440 ± 0.0097 / 0.0032 ± 0.0004 & 0.149 ± 0.0280 / 0.0098 ± 0.0019 & 0.16 / 22.03 & 1.00 / 1.00 \\ 
  50 & 0.60 & 0.00 & 0.50 & 0.0409 ± 0.0047 / 0.0033 ± 0.0001 & 0.130 ± 0.0256 / 0.0094 ± 0.0013 & 0.13 / 16.22 & 1.00 / 1.00 \\ 
  50 & 0.60 & 0.00 & 1.00 & 0.0462 ± 0.0046 / 0.0036 ± 0.0005 & 0.140 ± 0.0235 / 0.0105 ± 0.0018 & 0.09 / 2.49 & 1.00 / 1.00 \\ 
  50 & 0.60 & 0.60 & 0.10 & 0.0435 ± 0.0328 / 0.0027 ± 0.0007 & 0.108 ± 0.0485 / 0.0075 ± 0.0012 & 0.13 / 18.56 & 0.20 / 1.00 \\ 
  50 & 0.60 & 0.60 & 0.50 & 0.0448 ± 0.0148 / 0.0030 ± 0.0011 & 0.115 ± 0.0320 / 0.0082 ± 0.0019 & 0.10 / 11.41 & 1.00 / 1.00 \\ 
  50 & 0.60 & 0.60 & 1.00 & 0.2320 ± 0.2001 / 0.0036 ± 0.0017 & 0.359 ± 0.2532 / 0.0081 ± 0.0024 & 0.07 / 1.69 & 1.00 / 1.00 \\ 
  \hline
\end{tabular}
\caption{Results for Multivariate response linear regression models for $m=20,$ $n = 500/n = 100000$}
\label{tab:sal}
\end{table}

For all possible combinations of results, Table \ref{tab:sal}
provides the specific outcomes for $m = 20$ and $n = 500$/$n = 100,000$. We assume that, in the worst-case scenario for $m = 20$ and $n = 100,000$, the maximum average computation time is less than 20 seconds. In this case, the model successfully selects the correct variables, and the error in the $\beta$ coefficients is close to zero, indicating high statistical efficiency without sacrificing accuracy. 
For $n = 500$, the computation time is a fraction of a second. However, when the value of the $\beta$ coefficient is very small (e.g., $\texttt{effect} = 0.1$), the variable selection capability degrades. This effect is more pronounced in environments with high correlation between predictors and response ($\rho_{\mathcal{X}} = \rho_{\mathcal{Y}} = 0.6$). 
Complete results for the rest of cases with similar conclusions can be found in Appendix. 

\subsection{Probability Distributions with the $2-$Wasserstein Metric}

\begin{table}[ht]
\centering
\begin{tabular}{rrrrr}
\toprule
\textbf{n} & \textbf{m} & \textbf{dim(s)} & \textbf{Average Time} & \textbf{Proportion of selected the correct variable} \\
\midrule
200 & 50 & 1 & 0.01 & 1.00 \\
200 & 50 & 8 & 0.08 & 1.00 \\
200 & 150 & 1 & 0.11 & 1.00 \\
200 & 150 & 8 & 0.26 & 1.00 \\
200 & 300 & 1 & 0.19 & 1.00 \\
200 & 300 & 8 & 0.55 & 1.00 \\
2000 & 50 & 1 & 0.04 & 1.00 \\
2000 & 50 & 8 & 0.19 & 1.00 \\
2000 & 150 & 1 & 0.27 & 1.00 \\
2000 & 150 & 8 & 4.56 & 1.00 \\
2000 & 300 & 1 & 1.29 & 1.00 \\
2000 & 300 & 8 & 8.36 & 1.00 \\
20000 & 50 & 1 & 0.27 & 1.00 \\
20000 & 50 & 8 & 1.00 & 1.00 \\
20000 & 150 & 1 & 1.30 & 1.00 \\
20000 & 150 & 8 & 2.35 & 1.00 \\
20000 & 300 & 1 & 1.36 & 1.00 \\
20000 & 300 & 8 & 5.46 & 1.00 \\
100000 & 50 & 1 & 1.18 & 1.00 \\
100000 & 50 & 8 & 4.80 & 1.00 \\
100000 & 150 & 1 & 3.47 & 1.00 \\
100000 & 150 & 8 & 14.47 & 1.00 \\
100000 & 300 & 1 & 6.94 & 1.00 \\
100000 & 300 & 8 & 31.32 & 1.00 \\
\bottomrule
\end{tabular}
\caption{Summary of performance metrics for different parameter configurations.}
\label{tab:summary}
\end{table}

To illustrate the performance of the novel variable selection methods for functional data in the field of distributional data analysis, where the outcome is a probability distribution, we compare their performance with state-of-the-art methods \cite{tucker2021variable}. Similar to \cite{tucker2021variable}, we introduce a simulation example from \cite{alex2019wasserstein} with correlated scalar predictors $X_j \sim \mathcal{U}(-1,1)$, for $j = 1, \dots, p$, generated in two steps: 

\begin{enumerate}
    \item $Z = (Z_1, \dots, Z_p)^{\top}$ is a multivariate Gaussian with $\mathbb{E}(Z_j) = 0$ and $\text{Cov}(Z_j, Z_{j'}) = \rho^{|j-j'|}$.
    \item $X_j = 2\phi(Z_j) - 1$ for $j = 1, \dots, p$, where $\phi$ is the standard normal distribution function.
\end{enumerate}

We set $p = 10$ and $\rho = 0.5$. The Fréchet regression function is given by:
\[
m(x) = \mathbb{E}(Y(\cdot) \mid X = x) = \mu_0 + \beta x_4 + (\sigma_0 + \gamma x_4) \phi^{-1}(\cdot).
\]

Conditional on $X$, the random response $Y$ is generated by adding noise as follows: $Y = \mu + \sigma \phi^{-1}(\cdot)$ with $\mu \mid X \sim \mathcal{N}(\mu_0 + \beta X_4, v_1)$ and $\sigma \mid X \sim \gamma \left(\frac{(\sigma_0 + \gamma X_4)^2}{v_2}, \frac{v_2}{\sigma_0 + \gamma X_4}\right)$. These are independently sampled, where $\mathcal{N}(\cdot, \cdot)$ and $\gamma(\cdot, \cdot)$ denote Gaussian and gamma distributions, respectively. It is evident that only $X_4$ is important. 

The parameters were chosen as $\mu_0 = 0$, $\sigma_0 = 3$, $\beta = 3$, $\gamma = 0.5$, $v_1 = 1$, and $v_2 = 2$, following \cite{tucker2021variable}. In this setup, the competing algorithms almost always select the variable $X_4$, as does our model. Since we analyze the same parameter configuration and focus on the same effect size, the primary purpose of this analysis is to test the computational performance of our methods. We run the algorithm to find a total number of variables $\dim(s) \in \{1, 3, 5, 8\}$ to gain insight into how the computational cost increases and what can be expected in real-world scenarios.

From a theoretical standpoint, if $\dim(s)$ increases, the running time is expected to rise due to the increased complexity of the corresponding optimization problem. Additionally, we focus on different sample sizes $n \in \{200, 2000, 20000, 100000\}$ and varying numbers of points in the grid of quantile functions $m \in \{50, 100, 150, 200\}$. In the case where $n = 200$, the computation takes around 25 minutes on the laptop used for the experiments, whereas larger sample sizes take days. In this generative example, the methods always select the true variable. In each simulation, $b = 1, \dots, B = 200$, we record the execution time $t_b$. We also check if the model correctly identifies the true variable $X_4$, denoted as $M_{b, \text{true}}$, which is defined by the condition:

\[
M_{b, \text{true}} = \mathbb{I} \{\widehat{\beta}_{r4} \neq 0 \text{ for any } r \in \{1, \dots, m\} \}.
\]

Table \ref{tab:summary}
 confirms that the variable selection algorithm performs comparably to \texttt{FRISO} in selecting the correct variable for both small sample sizes ($n = 200$) and large sample sizes ($n = 1,000,000$). In the case of quantile outcomes defined on a grid of 300 points, the algorithm can complete the variable selection process in just 30 seconds, even when identifying up to $8$ variables. We must note that the computational limitations of \texttt{FRISO} that only can support datasets of less $n=300$ observations in less of hour time, limited exhaustive comparative with our proposal. With large sample size $n$, the statistical consistency results of our proposal garantee a good proposal in the practice.




\section{Discussion}\label{sec:discus}

This paper presents a novel variable selection framework for regression models in general spaces, such as negative type spaces—a setting increasingly common in modern healthcare applications for modeling complex representations of patients data \cite{matabuenacontributions}. We focus primarily on settings where the corresponding loss function $\ell(\cdot, \cdot)$ for model fitting is convex and the response lie in a Hilbert Spaces. The general mathematical formulation supports different regression and classification tasks for each coordinate of the response variable simultaneously, distinguishing it from existing methods in the literature. 


Empirical validation in various clinical applications and data structures, focusing on diabetes research, demonstrates the versatility and potential of  variable selection algorithm to analyze sophisticated complex biomarkers that appear in digital health. It illustrates the potential to use complex patient representations to build more accurate clinical profiles and understand patients' each metabolic processes. 


Our variable selection framework is particularly effective in convex subsets of separable Hilbert spaces, enabling efficient solutions through projection arguments. The high computational scalability of our method, demonstrated with datasets containing millions of patients, supports the integration of subsampling techniques such as stability selection and other hyperparameter-based inference techniques. This computational strength allows for statistical inference on variable selection in large-scale medical cohorts and serves as a robust alternative to popular variable selection methods in the statistical literature, including frequentist and Bayesian methods.

\bibliographystyle{unsrt}
\bibliography{referencias.bib/arxiv}

\appendix

\newpage

\newpage

\end{document}